%% file: main.tex
\documentclass{article} %
\usepackage{times}
    \usepackage[final]{neurips_2024}

\usepackage[utf8]{inputenc} % allow utf-8 input
\usepackage[T1]{fontenc}    % use 8-bit T1 fonts

\usepackage{xcolor}
\usepackage{hyperref}       %
\usepackage{url}

\usepackage{algorithm}
\usepackage[noend]{algpseudocode}
\usepackage{algorithmicx}
\usepackage{microtype}
\usepackage{lipsum}
\usepackage{graphicx}
\usepackage{subcaption}
\usepackage{sidecap}
\usepackage{caption}
\usepackage{booktabs} %
\usepackage{amsfonts}       %
\usepackage{nicefrac}       %
\usepackage{microtype}      %
\usepackage{comment}
\usepackage{enumitem}
\usepackage{wrapfig,tikz}
\usepackage{amsmath,longtable,fancyhdr}
\usepackage{bigstrut, tabularx, multirow, makecell, diagbox}
\usepackage[export]{adjustbox}
\usepackage{graphicx,float}
\usepackage{amssymb,amsthm, mathtools}
\usepackage{stackrel}
\usepackage{color}
\usepackage{colortbl}
\usepackage{theoremref}
\usepackage{cases}
\usepackage{stmaryrd}
\usepackage{mathabx}
\usepackage{dsfont}
\usepackage[capitalise]{cleveref}

\usepackage{mdframed}

\newtheorem{theorem}{Theorem}
\newtheorem*{theorem*}{Theorem}

\newtheorem{definition}{Definition}

\definecolor{Gray}{gray}{0.85}
\definecolor{LightCyan}{rgb}{0.88,1,1}
\definecolor{blue1}{RGB}{0,128,255}
\definecolor{blue3}{RGB}{0,0,128}
\definecolor{darkpastelgreen}{rgb}{0.01, 0.75, 0.24}
\definecolor{cerulean}{rgb}{0.0, 0.48, 0.65}

\newcommand{\bu}{\mathbf{u}}
\newcommand{\bv}{\mathbf{v}}

\newcommand{\bx}{\mathbf{x}}
\newcommand{\bz}{\mathbf{z}}
\newcommand{\by}{\mathbf{y}}
\newcommand{\bw}{\mathbf{w}}
\newcommand{\bfx}{\mathbf{f}(\mathbf{x})}
\newcommand{\bff}{\mathbf{f}}

\newcommand{\ud}{~\mathrm{d}}

\newcommand{\bbR}{\mathbb{R}}
\newcommand{\I}{\text{I}(\bx)}
\newcommand{\btheta}{\boldsymbol{\theta}}

\newcommand{\bphi}{\boldsymbol{\varphi}}
\newcommand{\bpsi}{\boldsymbol{\psi}}
\newcommand{\bphipx}{\boldsymbol{\varphi}^{\prime}(\mathbf{x})}
\newcommand{\bphipxi}{\boldsymbol{\varphi}^{\prime}(\mathbf{x}^i)}

\newcommand{\bg}{\mathbf{g}_{\boldsymbol{\theta}}}
\newcommand{\lx}{\mathcal{L}_{\mathbf{x}}}
\newcommand{\lxhat}
{\hat{\mathcal{L}}_{\mathbf{x}}}
\newcommand{\lz}{\mathcal{L}_{\mathbf{z}}}
\newcommand{\lxp}{\mathcal{L}_{\mathbf{x}, p}}
\newcommand{\tlx}{\tilde{\mathcal{L}}_{\mathbf{x}}}
\newcommand{\tlxp}{\tilde{\mathcal{L}}_{\mathbf{x},p}}
\newcommand{\norm}[1]{\lVert#1\rVert}

\newcommand{\bbE}{\mathbb{E}}

\makeatletter
\usepackage{xspace} 
\def\@onedot{\ifx\@let@token.\else.\null\fi\xspace}
\DeclareRobustCommand\onedot{\futurelet\@let@token\@onedot}

\def\eg{\emph{e.g}\onedot}

\def\ie{\emph{i.e}\onedot}

\def\wrt{w.r.t\onedot}

\def\iid{i.i.d\onedot}

\newcommand{\appcref}[1]{Appendix~\cref{#1}}

\title{Learning Macroscopic Dynamics from Partial Microscopic Observations}

\author{
  Mengyi Chen\textsuperscript{1}, ~Qianxiao Li\textsuperscript{1, 2}\\
  Department of Mathematics, National University of Singapore\textsuperscript{1},\\
  Institute for Functional Intelligent Materials, National University of Singapore\textsuperscript{2}\\
  \texttt{chenmengyi@u.nus.edu}, 
  ~~\texttt{qianxiao@nus.edu.sg}}

\begin{document}
\maketitle
\input{abstract}

\input{intro}
\input{related}

\input{setup}
\input{methods}
\input{experiments}
\input{conclusion}

\bibliography{main}
\bibliographystyle{iclr2024_conference}

\appendix
\input{appendix_0}
\input{appendix_proof}

\input{appendix_details}
\input{appendix_D}

% \input{checklist}

\end{document}

%% file: abstract.tex
\begin{abstract}
Macroscopic observables of a system are of keen interest in real applications such as the design of novel materials. Current methods rely on microscopic trajectory simulations, where the forces on all microscopic coordinates need to be computed or measured. However,  this can be computationally prohibitive for realistic systems.  In this paper, we propose a method to learn macroscopic dynamics requiring only force computations on a subset of the microscopic coordinates. Our method relies on a sparsity assumption: the force on each microscopic coordinate relies only on a small number of other coordinates. The main idea of our approach is to map the training procedure on the macroscopic coordinates back to the microscopic coordinates, on which partial force computations can be used as stochastic estimation to update model parameters. We provide a theoretical justification of this under suitable conditions. We demonstrate the accuracy, force computation efficiency, and robustness of our method on learning macroscopic closure models from a variety of microscopic systems, including those modeled by partial differential equations or molecular dynamics simulations. 
Our code is available at \url{https://github.com/MLDS-NUS/Learn-Partial.git}. 
% using a variety of latent models
\end{abstract}

%% file: intro.tex
\section{Introduction}
Macroscopic properties, including thestructural and dynamical properties, provide a way to describe and understand the collective behaviors of complex systems. 
In a wide range of real applications, researchers focus mainly on the macroscopic properties of a system, \eg, the viscosity and ionic diffusivity of liquid electrolytes for Li-ion batteries~\citep{dajnowicz2022high}. 
Macroscopic observables usually depend on the whole microscopic system, \eg, the calculation of mean squared displacement requires all microscopic coordinates during the simulation. With growing simulation and experimental data, data-driven learning of macroscopic properties from microscopic observations has become an active area of research ~\citep{zhang2018deepcg, wang2019machine, husic2020coarse, lee2020coarse, fu2023simulate, chen2024constructing}. 

Accurate calculation of macroscopic properties requires large-scale microscopic simulation.  
% In the microscopic simulation, the forces on all the microscopic coordinates need to be calculated and integrated each time step. 
However, accurate force computations on all microscopic coordinates for large systems are extremely expensive~\citep{jia2020pushing, musaelian2023scaling}. 
 For example, in \emph{ab initio} molecular simulations, accurate forces need to be calculated from density functional theory (DFT). The computational cost of DFT limits its application to relatively small systems, typically ranging from a few hundred atoms to several thousand atoms, depending on the level of accuracy and computation resources ~\citep{hafner2006toward, luo2020parallel}. 
This poses a dilemma: Accurate macroscopic properties are obtained from large-scale microscopic simulation, but the computation of forces on all the microscopic coordinates is extremely challenging. 

To solve the dilemma, the corresponding question is: 
Can we still obtain accurate macroscopic observables even though only access to forces on a subset of the microscopic coordinates?
In this work, we develop a method to learn the dynamics of the macroscopic observables directly, while only forces on a subset of the microscopic coordinates are needed. 
Efficient partial computation of microscopic forces relies on the sparsity assumption, where the computation cost of forces on a subset of microscopic coordinates does not scale with the microscopic system size.  
To learn the dynamics of the macroscopic observables, we first map the macroscopic dynamics back to the microscopic space, then compare it with the partial microscopic forces. 
Our key idea is summarized in \cref{fig:model}. 

Our main contributions are as follows: 
\begin{itemize}
    \item  We develop a novel method that can learn the macroscopic dynamics from partial computation of the microscopic forces. Our method can significantly reduce the computational cost for force computations. 
    \item We theoretically justify that forces on a subset of the microscopic coordinates can be used as stochastic estimation to update latent model parameters, even if the macroscopic observables depend on all the microscopic coordinates.
    \item We empirically validate the accuracy, force computation efficiency, and robustness of our method through a variety of microscopic dynamics and latent model structures. 
\end{itemize}

\begin{figure}
    \centering
    \includegraphics[width=1.0\linewidth]{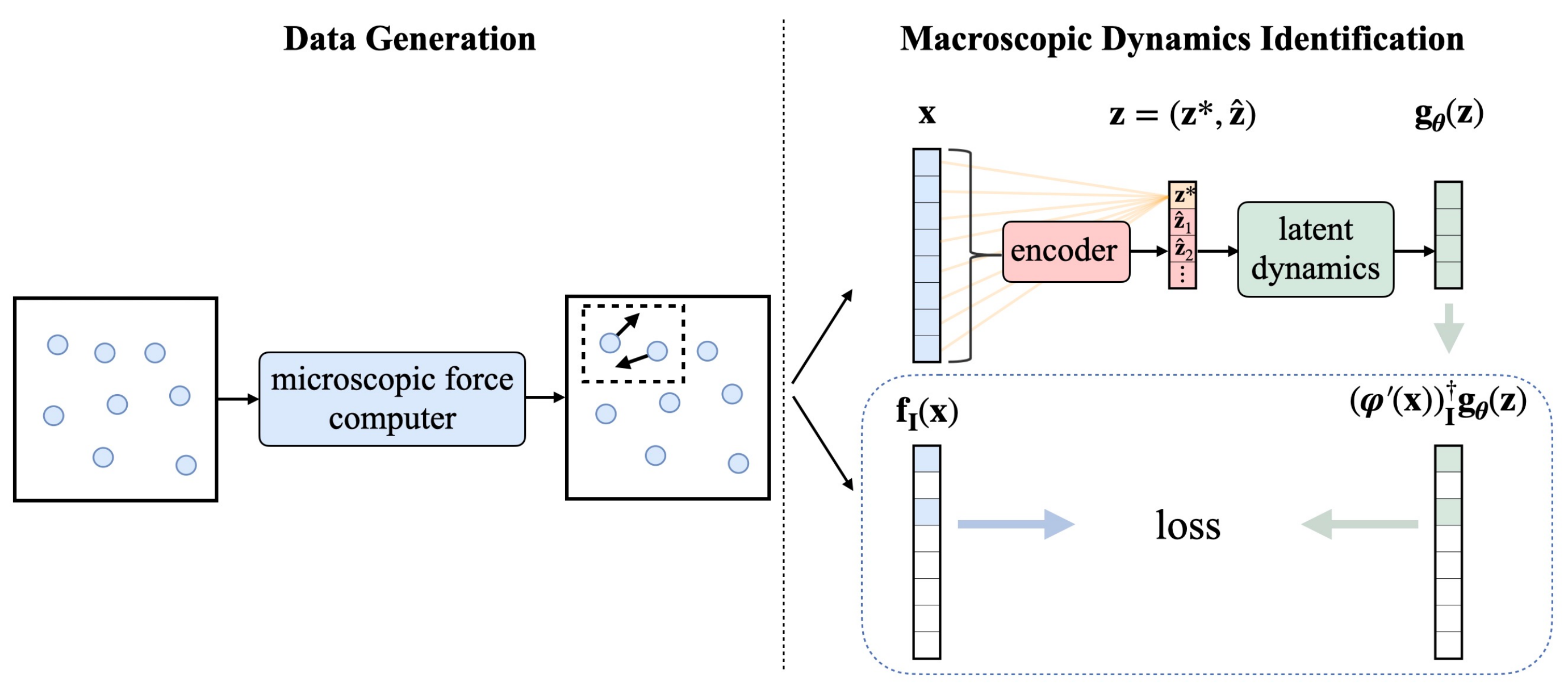}
    \caption{Overview of our method. \emph{Left}. Data generation workflow.  For each configuration $\bx$, forces on a subset of all the microscopic coordinates are calculated by the microscopic force calculator. \emph{Right}. Macroscopic dynamics identification. The macroscopic dynamics is mapped to the microscopic space first, then compared with the forces on a subset of the microscopic coordinates.}
    \label{fig:model}
    \vspace{-3mm}
\end{figure}

%% file: related.tex
\section{Related Work}

\paragraph{Learning from Partial Observations} 
 Several works have sought to learn dynamics from partially observed state $\hat{\bx}$ utilizing machine learning~\citep{
ruelle1971nature,sauer1991embedology,takens2006detecting,ayed2019learning, ouala2020learning, huang2020learning, schlaginhaufen2021learning, lu2022discovering, stepaniants2023discovering}. 
For training, these methods reconstruct the unobserved state $\mathbf{\Tilde{x}}$ first and model the dynamics of $\mathbf{x}=(\mathbf{\hat{x}, \Tilde{x}})$. 
Our work assumes full state $\mathbf{x}$ but partial forces $\mathbf{f}$. 
Furthermore, we do not model the dynamics on state $\mathbf{x}$ directly, but rather on the latent space because the dimension of $\mathbf{x}$ would be extremely high for large systems.

\paragraph{Reduced Order Models}  
By modeling the dynamics in the latent space and then recovering the full states from them, 
reduced order models (ROMs) substitute expensive full order simulation with cheaper reduced order simulation~\citep{schilders2008model, fresca2020deep, lee2020model, hernandez2021deep, fries2022lasdi}. 

Our method can be thought to fall in the range of closure modeling. 
Unlike ROMs, we aim to model the dynamics of some given macroscopic observables directly and are not interested in recovering the microscopic states from the latent states. 

\paragraph{Equation-free Framework}
The equation-free framework (EFF) has sought to simulate the macroscopic dynamics efficiently~\citep{kevrekidis2003equation,samaey2006patch, liu2015acceleration}. 
The EFF is usually applied to partial differential equation (PDE) systems, and the macroscopic observables are chosen to be the solution of the PDE at the coarse spatial grid. 
In EFF, the macroscopic observables depend locally on the microscopic coordinates, allowing the macroscopic dynamics to be directly estimated from the microscopic simulations performed in small spatial domains. In contrast, the macroscopic observables may depend globally on the microscopic coordinates in our method, and the macroscopic dynamics may not be easily estimated from microscopic simulations performed in small spatial domains.

% Another difference is that EFF can bypass explicit derivation of macroscopic evolution law by coupling microscale and macroscale dynamics, while our method explicitly learns the macroscopic dynamics. 
% During simulation, EFF still requires microscopic simulation to be performed in small spatial domains and for short times, while our method can directly simulate the macroscopic dynamics without performing any microscopic simulation.

Another difference is that our method explicitly learns the macroscopic dynamics, while EFF can bypass explicit derivation of macroscopic evolution law by coupling microscale and macroscale dynamics. During simulation, EFF still requires microscopic simulation to be performed in small spatial domains and for short times, but our method can enable fast macroscopic simulation without requiring any microscopic simulation
% However, the learned dynamics may involve approximation or statistical errors that are often challenging to estimate. 
However, for systems where the macroscopic evolution equations conceptually exist but are not available in closed form, EFF can efficiently handle such cases, but the learned dynamics in our method may involve approximation or statistical errors that are often challenging to estimate.

% \paragraph{Discovery of Macroscopic Dynamics}
% Several works are also interested in modeling the macroscopic dynamics using neural networks, but they require microscopic trajectories or forces on all the microscopic coordinates for training~\cite {champion2019data, fu2023simulate, chen2024constructing}. 

%% file: setup.tex
\section{Problem Setup}
We consider a microscopic system consisting of $n$ particles.
Let the state of the microscopic system be $\bx = (\bx_1, \cdots, \bx_n)\in \bbR^N,  
\bx_i\in \bbR^m, N = mn$, where $\bx_i\in\bbR^m$ is some physical quantity associated with the $i$-th particle, such as the position and velocity. Assume the dynamics of the microscopic system can be characterized by an ordinary differential equation(ODE):  
\begin{align}\label{eq:ODE}
    \frac{\ud  \bx(t)}{\ud t} = \mathbf{f}(\bx(t))
\end{align}
where $\mathbf{f}(\bx) = (\mathbf{f}_1(\bx), \cdots, \mathbf{f}_n(\bx)) \in \bbR^N$. We will call $\bx_i$ the \emph{microscopic coordinate} of the $i$-th particle and $\bff_i$ the \emph{force} acting on the microscopic coordinate $\bx_i$. 
In many real applications, we are interested in the dynamics of some macroscopic observables $\bz^{\ast} = \bphi^{\ast}(\bx)$. Here $\bphi^{\ast}$ is given beforehand and describes the functional dependence of $\bz^{\ast}$ on $\bx$. For example, $\bz^{\ast}$ can be chosen to be the instantaneous temperature or mean squared displacement in a Lennard-Jones system. 

The goal is to learn the dynamics of $\bz^{\ast}$ from microscopic simulation data. Existing methods that try to learn the macroscopic or latent dynamics require microscopic trajectories or forces on all the microscopic coordinates for training~\citep{champion2019data, fries2022lasdi, fu2023simulate, chen2024constructing}. 
The problem is: When the microscopic system size $N$ is very large such that the force computations on all the microscopic coordinates are impossible, these methods are no longer applicable. Instead, our method aims to learn from partial computation of microscopic forces. 

Consider we are given a microscopic force calculator $\mathcal{S}$ for computation of partial microscopic forces. Let the microscopic coordinate $\bx$ be sampled from a distribution $\mathcal{D}$. For each $\bx$, the microscopic force calculator $\mathcal{S}$ will first sample an $n$ dimensional random variable $\mathbf{I}(\bx) = (\mathbf{I}_{1}(\bx), \cdots, \mathbf{I}_{n}(\bx)) \sim \mathcal{P}_{\bx} \in \{0,1\}^n$ according to a certain strategy. Next $\mathcal{S}$ will calculate the corresponding \emph{partial forces} $\bff_{\mathbf{I}(\bx)} \coloneqq (\mathbf{f}_{\mathbf{I}_{1}(\bx)}, \cdots, \mathbf{f}_{\mathbf{I}_{n}(\bx)})$. For notation simplicity sometimes we will simply write $\bff_{\mathbf{I}(\bx)}$ as $\bff_{\mathbf{I}}$. $\mathbf{I}_i(\bx)$ indicate whether partial $i$ is chosen to calculate the force or not. 
If particle $i$ is chosen, $\mathbf{I}_{ i}(\bx) = 1, \mathbf{f}_{\mathbf{I}_{i}(\bx)} = \mathbf{f}_{i}$, otherwise  $\mathbf{I}_{i}(\bx) = 0, \mathbf{f}_{\mathbf{I}_i(\bx)} = \mathbf{0}$.
We require the sampling strategy $\mathcal{P}_{\bx}$ to satisfy: 
\begin{enumerate}
    \item For each $\mathbf{I}({\bx})\sim \mathcal{P}_{\bx}$, exactly $n\cdot p$ items are equal to 1 and the rest are 0. 
    \item Each particle can be chosen with equal probability p, \ie $\mathbb{P}(\mathbf{I}_{i}(\bx) = 1) = p, \mathbb{P}(\mathbf{I}_{i}(\bx) = 0) = 1 - p, i=1,\cdots, n$. 
\end{enumerate}
This means that the microscopic force calculator $\mathcal{S}$
can calculate forces on  $n\cdot p$ microscopic coordinates, and $0<p<1$ limits the computation capacity of the microscopic force calculator $\mathcal{S}$.   
The above requirement is consistent with real applications since it is difficult to calculate all the microscopic forces due to computational cost. 
Furthermore, for efficient calculation of the partial microscopic forces, we will  assume $\bff$ satisfies the following sparsity assumption:

\textbf{Assumption}:  
For a given error tolerance $\epsilon > 0$,  
there exists a constant $M\ll n$, such that for any $\bx \sim \mathcal{D}$ and $i\in \{1, \cdots, n\}$, we can always find an index set $J(\bx_i)\subset\{i=1,\cdots,n\}, |J(\bx_i)| < M$  which satisfies:
\begin{align}
    ||\bff_i(\bx_1, \cdots, \bx_n) - \Tilde{\bff_i}(\{\bx_i\}_{i\in J(\bx_i)}) ||_2 < \epsilon
\end{align}
Intuitively, the assumption implies that the computational cost of force $\mathbf{f}_i$ is independent of the microscopic system dimension $N$. Thus our microscopic force calculator $\mathcal{S}$ can compute partial forces in an efficient way. This assumption is prevalent in real-world applications. To better illustrate this, we give two examples here. 
The first example is about molecular dynamics. In molecular dynamics,  each $\bx_i$ represents the position $\mathbf{r}_i$ and velocities $\mathbf{v}_i$ of the $i$-th atom, \ie,  $\bx_i = (\mathbf{r}_i, \mathbf{v}_i)\in\bbR^6$. Then \cref{eq:ODE} becomes the Newton's law of motion:
    \begin{align}
        \frac{\ud \mathbf{r}_i}{\ud t} &= \mathbf{v}_i \\
        \frac{\ud \mathbf{v}_i}{\ud t} &= \frac{1}{m_i}  \mathbf{F}_i(\mathbf{r}_1, \cdots, \mathbf{r}_n), 
    \end{align}
    It is common to limit the range of pairwise interactions to a cutoff distance~\citep{allen2004introduction, ZHOU202241, vollmayr2020introduction}. 
To calculate the force on an atom, we only need to consider its interaction with other atoms that are within the cutoff. 
The second example is about systems modeled by partial differential equation (PDE).
We consider a time-dependent PDE and apply finite difference scheme to discretize the spatial derivatives. Then, the resulting semi-discretized equation takes the form of \cref{eq:ODE}, and each $\bx_i$ is the value at the $i$-th grid. $\bff_i$ only depends on those grids that are used for finite difference approximation of the spatial derivatives.

Let the training data generated by the microscopic force calculator $\mathcal{S}$ be $\{\bx^i, \bff_{\mathbf{I}(\bx^i)}\}_{i=1,\cdots, K}$. 
The data generation procedure is provided in \cref{alg:data}. 
We will introduce in the next section how we can learn the macroscopic dynamics from the training data with partial forces. 

%% file: methods.tex
\section{Method}\label{sec:method}
Existing works for macroscopic dynamics identification consist of two parts: dimension reduction and macroscopic dynamics identification~\citep{fu2023simulate, chen2024constructing}. We will follow these two parts. 
We start with most standard parts of closure modeling with an autoencoder, next, we turn to the main difficulty of macroscopic dynamics identification from partial forces.

\subsection{Autoencoder for Closure Modeling}
We will use an autoencoder to find the closure $\hat{\bz} = \hat{\bphi}(\bx)$ to $\bz^{\ast} = \bphi^{\ast}(\bx)$ such that $\bz = (\bz^{\ast}, \hat{\bz}) $ forms a closed system. Here we define $\bz$ as forming a closed system if its dynamics
$\dot{\bz}$ depends only on $\bz$, not any external variables. 
Note that in $\bz^{\ast } = \bphi^{\ast}(\bx)$, $\bphi^{\ast}$ is determined beforehand and contains no trainable parameters. This ensures $\bz^{\ast}$ represents the desired macroscopic observables and remains unchanged during the training of the autoencoder.

Denote the encoder by $\bphi = (\bphi^{\ast}, \hat{\bphi})$ and the decoder by $\bpsi$, 
we will minimize the following reconstruction loss: 
\begin{align}
    \mathcal{L}_{\text{rec}} = \textstyle \frac{1}{K}\sum_{i=1}^K \norm{\bx^i - \bpsi \circ \bphi(\bx^i)}_2^2
\end{align}
We also want $\bphipx\bphipx^T$ to be well-conditioned~(see \cref{sec:4.3}), then we 
impose constraints on the condition number of $\bphipx\bphipx^T$:
\begin{equation}
\begin{aligned}\label{eq:cond}
    \mathcal{L}_{\text{cond}} &= \textstyle \frac{1}{K} \sum_{i=1}^K \norm{\kappa(\bphipxi\bphipxi^T) -1}_2^2 \\
\end{aligned}\end{equation}
By enforcing $\bphipxi \bphipxi^T$ to be well-conditioned, we are also enforcing $\bphipxi\in \bbR^{d\times N}, d\ll N$ to have full row rank, which will be used later. The overall loss to train the autoencoder is :
\begin{align}\label{eq:AE_loss}
     \mathcal{L}_{\text{AE}} =  \mathcal{L}_{\text{rec}} + 
     \lambda_{\text{cond}} \mathcal{L}_{\text{cond}}, 
\end{align}
$\lambda_{\text{cond}} $ is a hyperparameter to adjust the ratio of $\mathcal{L}_{\text{cond}}$ and is chosen to be quite small in our experiments, \eg, $10^{-5}$ or $10^{-6}$. 
The aim of training the decoder $\bpsi$ is to help the discovery of the closure variables. The decoder $\bpsi$ will not be used for further macroscopic dynamics identification. To facilitate comparison between models tainted with all and partial forces, we will train the autoencoder first and freeze it for macroscopic dynamics identification. 

\subsection{Macroscopic Dynamics Identification}\label{sec:4.3}
We now address the difficulty of learning from data with partial microscopic forces. 
Substitute $\mathbf{z} = \bphi(\mathbf{x})$ into equation \cref{eq:ODE} and make use of chain rule, we get the dynamics of $\bz$:
\begin{align}\label{eq:dz}
    \frac{\mathrm{d}\mathbf{z}}{\mathrm{d} t} = \bphi^{\prime}(\mathbf{x}) \mathbf{f}(\mathbf{x}), \quad \bz(0) = \bphi(\bx_0)
\end{align}
here we use $\bphi^{\prime}(\mathbf{x})$ to denote the Jacobian of $\nabla_{\mathbf{x}}\bphi(\mathbf{x})$ for notation simplicity.
If the dynamics of $\bz$ is closed, the right-hand side of \cref{eq:dz} will only depend on $\bz$, and we parametrize it using a neural network $\bg (\bz) \approx \bphi^{\prime}(\mathbf{x}) \mathbf{f}(\mathbf{x}) $. 
% The goal is to reproduce the dynamics of the macroscopic observables $\bz^{\ast}$ with $\bg$:
% \begin{align}
%     \bz(t) &= \bz_0 + \int_{0}^t \bg(\bz(s)) ds  \approx \bphi(\bx(t))
% \end{align}
% (6) to (8)
% Once this goal is achieved, given any initial $\bx(0)$, we can encode $\bx(0)$ to $\bz(0)$ and integrate the evolution law of $\bz$ directly in the macroscale. 
Since we are only interested in macroscopic dynamics identification, the loss would be naturally defined on the macroscopic coordinates: 
\begin{align}\label{eq:lz}
    \textstyle \mathcal{L}_{\mathbf{z}}(\btheta) = \frac{1}{K} \sum_{i=1}^K ||\bphi^{\prime}(\mathbf{x}^i)\mathbf{f}(\mathbf{x}^i)  -\mathbf{g_{\theta}}(\bz^i)||^2 
\end{align}
\cref{eq:lz} is used commonly in existing work~\citep{champion2019data, fries2022lasdi, bakarji2022discovering, park2024tlasdi}. 

\begin{figure}[!t]
\begin{minipage}[!t]{0.40\linewidth}
\begin{algorithm}[H]
	\caption{Data generation.}
	\label{alg:data}
	\begin{algorithmic}[1]
	    \Require{ \Statex $\mathcal{D}$:configuration distribution \Statex  $\mathcal{S}$: microscopic force calculator \Statex $K$: training data size \Statex $\mathcal{P}$: partial index sampling strategy}
         \For{$i=1$ {\bfseries to} $K$}
            \State  $\bx^i \sim \mathcal{D}$
            \State  $\mathbf{I}(\bx^i) \sim \mathcal{P}_{\bx^i}$
            \State calculate $\mathbf{f}_{\mathbf{I}(\bx^i)}(\bx^i)$ using $\mathcal{S}$
        \EndFor
        % \item[]
        \Return{$\{\bx^i, \mathbf{f}_{\mathbf{I}(\mathbf{\bx^i})}(\bx^i)\}_{i=1,\cdots,K}$}
        \vspace{1mm}
	\end{algorithmic}
\end{algorithm}
\end{minipage}
% \hfill
\begin{minipage}[!t]{0.56\linewidth}
\begin{algorithm}[H]
	\caption{Training procedure.}
	\label{alg:training}
	\begin{algorithmic}[1]
	    \Require{ \Statex $\{\bx^i, \mathbf{f}_{\mathbf{\bx^i}}(\bx^i)\}_{i=1,\cdots,K}$}: data
      \Statex $B$: minibatch size
       \Statex $\btheta_0$: model parameter
       \Statex opt: optimizer
	    \While{stopping criterion is not met}
                \State{sample $J$ $\subset \{1,\cdots, K\}, |J| = B$}
                \State{Calculate $\lxp$ in \cref{eq:lxp} with $\{\bx^i, \mathbf{f}_{\mathbf{\bx^i}}(\bx^i)\}_{i\in J}$} 
                \State{$\btheta_{t+1} \gets$ opt($\btheta_t$, $\nabla \lxp$)}
        \EndWhile
        \Return{model parameter}
        \vspace{2mm}
	\end{algorithmic}
\end{algorithm}
\end{minipage}
\end{figure}

The main difficulty of $\mathcal{L}_{\mathbf{z}}$ is that it includes the matrix-vector product $\bphi^{\prime}(\mathbf{x})\mathbf{f}(\mathbf{x})$.
Note that the $i$-th entry of $\bphi^{\prime}(\mathbf{x})\mathbf{f}(\mathbf{x})$ can be written as  $\sum_{j=1}^n \bphi^{\prime}_{ij}(\mathbf{x})  \mathbf{f}_j(\mathbf{x})$, and it is difficult to find an unbiased estimation of the $i$-th entry using a subset of $\{\mathbf{f}_j(\bx)\}_{j=1,\cdots, N}$. Thus the accurate calculation of $\lz$ requires the forces $\{\mathbf{f}_j(\bx)\}_{j=1,\cdots, N}$  on all the microscopic coordinates. 

The main idea of our method is to map the loss on the macroscopic coordinates back to the microscopic coordinates:
\begin{align}\label{eq:lx}
    \textstyle\mathcal{L}_{\mathbf{x}}(\btheta) =  \frac{1}{K} \sum_{i=1}^K  ||\mathbf{f}(\mathbf{x}^i)  - (\bphi^{\prime}(\mathbf{x}^i))^{\dagger}\mathbf{g_{\theta}}(\mathbf{z}^i)||^2 
\end{align}

where $(\bphi^{\prime}(\mathbf{x}^i))^{\dagger} \in \mathbb{R}^{N\times d}$ is the Moore-Penrose inverse. Since $\bphi^{\prime}(\mathbf{x}^i)$ is of full row rank, $(\bphi^{\prime}(\mathbf{x}^i))^{\dagger}$ is in fact the right inverse of $\bphi^{\prime}(\mathbf{x}^i)$, \ie, $\bphi^{\prime}(\bx^i) (\bphi^{\prime}(\mathbf{x}^i))^{\dagger}$ is an identity matrix. Below we will show our  main theoretical result: 
\begin{theorem}\label{thm:1}
    Assume for any $\bx\sim \mathcal{D}$, the eigenvalues of $\bphipx\bphipx^T$ are lower bounded by $b_1$ and upper bounded by $b_2$, $0<b_1\leq b_2$. Then:
    \begin{align}
    b_1 (\lx(\btheta) + C) & \leq \lz(\btheta) \leq b_2 (\lx(\btheta) +C)
    \end{align}
    here C does not depend on $\btheta$ hence does not affect the optimization.  
\end{theorem}
The proof relies on singular value decomposition of $\bphi^{\prime}(\bx)$ and we provide the full proof in \cref{Appendix:thm1-proof}. 
\cref{thm:1} states that by minimizing $\lx(\btheta)$, 
we are actually narrowing the range of $\lz(\btheta)$. Hence we want $b_1$ and $b_2$ to be as close as possible, this is the reason why we constrain the condition number of $\bphipx\bphipx^T$ in \cref{eq:cond}.  
In the very special case where $b_1 = b_2$, minimizing $\lx(\btheta)$ is just equivalent to minimizing $\lz(\btheta)$. 

Note that in loss $\lx$, the term $\textstyle ||\mathbf{f}(\mathbf{x})  - (\bphi^{\prime}(\mathbf{x}))^{\dagger}\mathbf{g_{\theta}}(\mathbf{z})||$ can be rewritten as $\sum_{j=1}^n ||\mathbf{f}_j(\mathbf{x})  - (\bphi^{\prime}(\mathbf{x}))^{\dagger}_j\mathbf{g_{\theta}}(\mathbf{z})||$, and $\frac{1}{p}\sum_{j\in \mathbf{I}(\bx)} ||\mathbf{f}_j(\mathbf{x})  - (\bphi^{\prime}(\mathbf{x}))^{\dagger}_j\mathbf{g_{\theta}}(\mathbf{z})||$ can be regarded as its unbiased stochastic estimation.
% Note that by converting $\lz(\btheta)$ to $\lx(\btheta)$, we get rid of the matrix-vector product  $\bphi^{\prime}(\mathbf{x})f(\mathbf{x})$. 
Then, we can introduce our loss defined for partial forces:
\begin{align}\label{eq:lxp}
    \lxp(\btheta) &=   \textstyle \frac{1}{pK} \sum_{i=1}^K \norm{\mathbf{f}_{\textcolor{red}{\mathbf{I}(\bx^i)}}(\bx^i)  - (\bphi^{\prime}(\bx^i))^{\dagger}_{\textcolor{red}{\mathbf{I}(\bx^i)}} \bg(\bz^i)}_2^2 
 \end{align}
Here a constant $1/p$ is multiplied to $\lxp$ to guarantee: 
\begin{align}\label{eq:equal}
    \bbE_{\bx^1, \cdots, \bx^K} \bbE_{\mathbf{I}(\bx^1), \cdots, \mathbf{I}(\mathbf{x}^K)}\lxp(\btheta)  = \bbE_{\bx^1, \cdots, \bx^K} \lx(\btheta)
\end{align}
We provide the full proof of \cref{eq:equal} in \cref{Appendix:B.2}. 
By training with the model with $\lxp$, we can use data with partial forces as stochastic estimation to update model parameters. The full training procedure is provided in \cref{alg:training}. 

Note that in \cref{alg:data} during the data generation, $\mathbf{I}(x^i)$ is also sampled from its distribution. Thus, $\lxp$ is deterministic once the samples $\{\bx^i, \mathbf{f}_{\mathbf{I}(\mathbf{\bx^i})}(\bx^i)\}_{i=1,\cdots,K}$ are generated. 
In the limit, the estimation is unbiased:
\begin{theorem}[informal]\label{thm:2}
    Let $ \tlx(\btheta) = \bbE \lx(\btheta),   \btheta^{\ast}\in \arg \min_{\btheta} \tlx(\btheta)$, $\btheta_{K, p} \in  \arg \min_{\btheta}\lxp(\btheta)$, then under certain conditions: 
    \begin{align}\begin{aligned}
    \tlx(\btheta_{K, p}) &-  \tlx(\btheta^{\ast}) \overset{a.s.}{\longrightarrow} 0 \\
    \end{aligned}\end{align}
\end{theorem}
The proof utilized Rademacher complexity and a crucial assumption used in the proof is the uniform boundedness of $\lxp$. 
The formal version of \cref{thm:2} and the complete proof is provided in  \cref{Appendix:B.3}. 
\cref{thm:2} theoretically justifies the expected risk $\tlx(\btheta_{K, p})$ at the optimal parameter found by $\lxp$, converges to the optimal expected risk $\tlx(\btheta^{\ast}) $ as $K$ goes to infinity.

%% file: experiments.tex
\section{Experiments }\label{sec:experiments}
In this section, we experimentally validate the accuracy, force computation efficiency,  and robustness through a variety of microscopic dynamics. 
\subsection{Force Computation Efficiency}\label{sec:4.1}

We first consider a one-dimensional spatiotemporal Predator-Prey system, mainly to validate the correctness and force computations efficiency of our method. 
\paragraph{Predator-Prey System}
The simplified form of the Predator-Prey system with diffusion ~\citep{murray2003multi} is
\begin{equation}\label{eq:Predator-Prey}
    \begin{aligned}
        \frac{\partial u}{\partial t} &= u(1-u-v) + D\frac{\partial^2 u}{\partial x^2}  \\
         \frac{\partial v}{\partial t} &= av(u-b) + \frac{\partial^2 v}{\partial x^2}, \quad x\in \Omega = [0,1],\quad t\in[0,\infty)\\
    \end{aligned}
\end{equation}
where $u, v$ denote the dimensionless populations of the prey and predator respectively, $a, b, D$ are three parameters. The complex dynamics of Predator-Prey interaction, including the pursuit of the predator and the evasion of the prey in ecosystems, can be described by \cref{eq:Predator-Prey}. 

% \paragraph{Microscopic coordinates}
We discretize the spatial domain of \cref{eq:Predator-Prey} into 50 uniform grids with $x_i = (i-\frac{1}{2})\Delta x,  \Delta x = 0.02,  1\leq i \leq 50$. Let $\mathbf{u}(t) = (u(x_1, t), \cdots, u(x_{50}, t)), \mathbf{v}(t) = (v(x_1, t), \cdots, v(x_{50}, t))$, then $(\mathbf{u}(t), \mathbf{v}(t)) \in \bbR^{100}$ are treated as the microscopic states. After approximating the spatial derivatives in \cref{eq:Predator-Prey} with the finite difference method, we consider the semi-discrete equation which is an $N=100$ dimensional ODE as our microscopic evolution law ~(see \cref{Appendix:C.1}).  
\begin{wrapfigure}[16]{R}{0.45\textwidth}
\centering
\vspace{-.7cm}%
\hspace{1cm}%
\begin{tabular}{c}
\setlength{\tabcolsep}{0pt}
\includegraphics[width=0.43\textwidth]{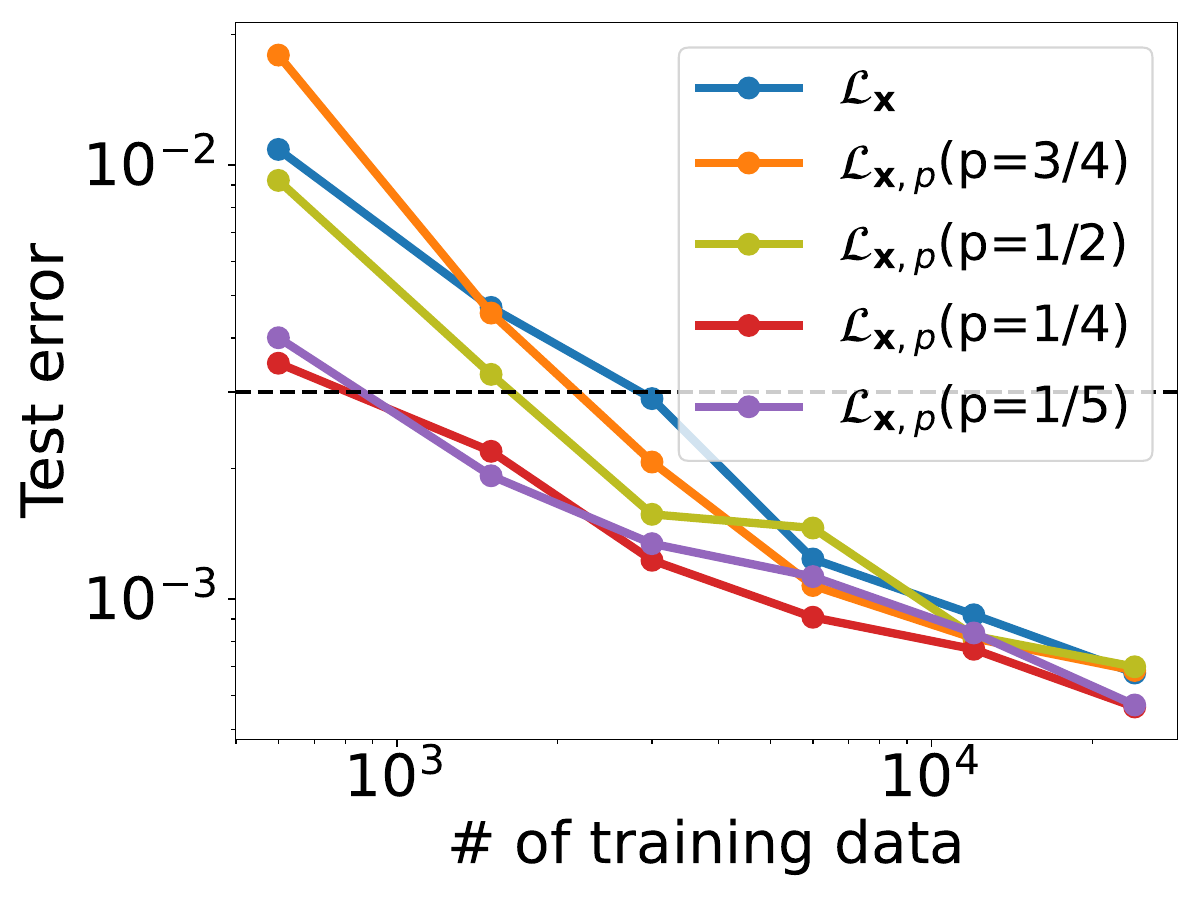}
\end{tabular}
\caption{
  Mean relative error on the test dataset of the Predator-Prey system. The black dashed line represents test error $=3\times 10^{-3}$. 
}
\label{fig:P-P}
\end{wrapfigure}

We choose the macroscopic observable of interest to be $\bz^{\ast}= (\Bar{u}, \Bar{v})$, the spatial average of the predator and the prey's population:
\begin{equation}
    \Bar{u} = \frac{1}{50}\sum_{i=1}^{50} u(x_i, t),\quad \Bar{v} = \frac{1}{50}\sum_{i=1}^{50} v(x_i, t)
\end{equation}
We find another 2 closure variables using the autoencoder, then the total dimension of the latent space $\bz$ is 4. We choose $\mathcal{D}$ to be the trajectory distribution of the state $\bx$. For the data generation with partial forces,  given $\bx$ we randomly choose forces on  $100\cdot p$ microscopic coordinate. For example, if $p=1/5$, then for each configuration, forces on 20 coordinates are calculated for training.

\paragraph{Results}
\cref{fig:P-P} shows the results on the test dataset which consists of 100 trajectories ~(for more detials, see Appendix \cref{tab:P-P}).
\emph{For models trained with partial microscopic forces, we report the equivalent number of training data with full forces throughout the paper.}  
 For example, for a model trained with $\lxp(p=1/5)$ on $3\times 10 ^3$ data, we will report the number of training data to be $3\times 10^3 \times 0.2 = 600$. 
The test error is defined to be the mean relative error of the macroscopic observables between the ground truth and the predicted trajectories as in \appcref{eq:error}. 

First, we observe that the mean relative error of all the models is around $10^{-4}$ when the number of training data is large enough. This tells us that training with $\lxp$ is correct and accurate, which is consistent with \cref{thm:2}.
The predicted trajectories fit quite well with the ground truth trajectories ~(see \appcref{fig:predator_prey_tra_0.1} and \cref{fig:predator_prey_tra_0.5}). 

We can conclude from \cref{fig:P-P} that, under the same number of training data, $\lxp$ with smaller $p~(1/4, 1/5)$ performs better. Similarly, to achieve the same performance, $\lxp$ with smaller $p$ requires less training data. We set the error tolerance to be $e_{\text{tol}}=3\times 10 ^{-3}$ and investigate how much training data is required to reach the error tolerance. In \cref{fig:P-P} the $x$-coordinate of the intersection point between the black dashed line and the other curves indicates the minimum data size required
If we arrange each model according to their minimal required training data, then $\lxp(p=1/5)\approx \lxp(p=1/4) < \lxp(p=1/2) <\lxp(p=3/4) <\lx$. 
Model trained with $\lxp(p=1/4,1/5)$ requires less data to reach $e_{\text{tol}}$, or equivalently, less force computations. This validates the \emph{force computation efficiency} of our method.  One explanation could be that there are many redundant information in the forces acting on all the microscopic coordinates. By using partial microscopic forces, $\lxp$ can explore more configurations $\bx$ given the same size of training data, thus can make use of more useful information. Another observation from \cref{fig:P-P} is that as the training data size increases, the gap between models trained with different $p$ narrows down. This is because as more data are provided, these data can contain almost all the information of the Predator-Prey system, thus more information will not lead to significant improvement. 

\subsection{Robustness to Different Latent Structures}\label{sec:4.2}
%   \begin{wraptable}[11]{r}{0.45\linewidth}
%     \definecolor{h}{gray}{0.9}
% 	\vspace{-2mm}
% 	\caption{Results on the Lennard-Jones system with 800 atoms and $N=4800$. Forces on 50 atoms are used to train $\lxp$ for all the latent model structures. Mean and standard deviation are reported over ten repeats. }
% 	\label{tab:LJ}
% 	\centering
% 	\scalebox{0.8}{
% 		\begin{tabular}{l c c}
%             \Xhline{3\arrayrulewidth}\bigstrut\bigstrut
%               & $\lz$  & \cellcolor{h}$\lxp(p =1/16)$ \\
%             \Xhline{1\arrayrulewidth}\bigstrut\bigstrut
%             MLP  & 2.60 \scalebox{0.6}{$\pm$ 1.01} $\times 10^{-2}$ & \textbf{1.34}\scalebox{0.6}{$\pm$ 0.26}$\mathbf{\times 10^{-3}}$  \\
%             OnsagerNet  & 4.45 \scalebox{0.6}{$\pm$ 2.03} $\times 10^{-3}$ & \textbf{1.17}\scalebox{0.6}{$\pm$ 0.18}$\mathbf{\times 10^{-3}}$  \\
%             GFINNs  & 1.28 \scalebox{0.6}{$\pm$ 0.74} $\times 10^{-2}$ & \textbf{2.84}\scalebox{0.6}{$\pm$ 0.99}$\mathbf{\times 10^{-3}}$  \\
%             \bottomrule
%             \end{tabular}
% 	}
% \end{wraptable}
\begin{wrapfigure}[17]{R}{0.45\textwidth}
\centering
\vspace{-.7cm}%
\hspace{1cm}%
\begin{tabular}{c}
\setlength{\tabcolsep}{0pt}
\includegraphics[width=0.43\textwidth]{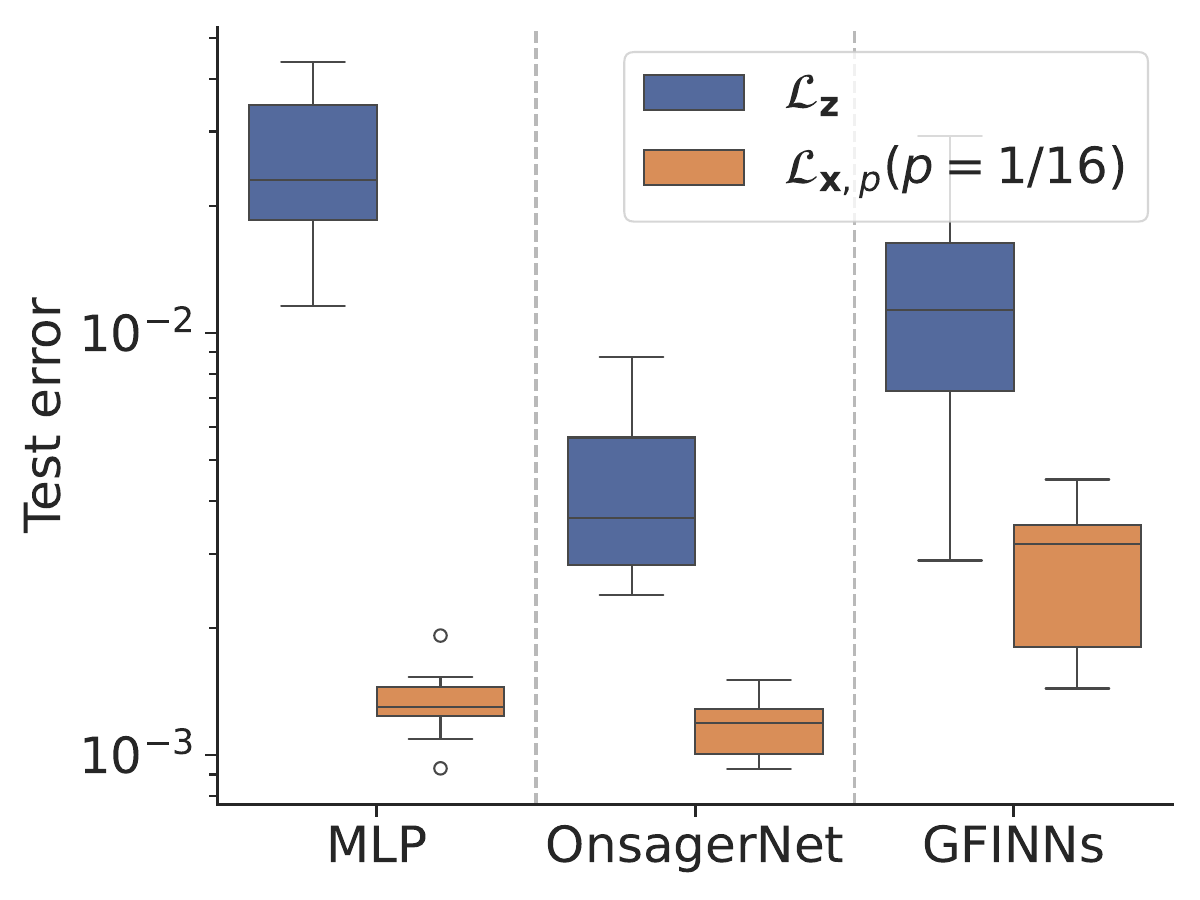}
\end{tabular}
\caption{
   Results on the Lennard-Jones system with 800 atoms and $N=4800$. Forces on 50 atoms are used to train $\lxp$ for all the latent model structures. Each model is trained with ten repeats.
}
\label{fig:LJ}
\end{wrapfigure}
Having validated the correctness and force computation efficiency of our method, we are ready to apply our method to a variety of latent structures. We will tackle the Lennard-Jones system in this subsection, and validate the robustness of our method to different latent model structures. 
Three latent model structures are considered: MLP,  OnsagerNet~\citep{yu2021onsagernet},  GFINNs~\citep{zhang2022gfinns}. Both the OnsagerNet and GFINNs endow the latent dynamical model with certain thermodynamic structure to ensure stability and interpretability. 
The specific implementations of these two models are slightly different. 

\paragraph{Lennard-Jones System}
The Lennard-Jones system is widely used in molecular simulation to study phase transition, crystallization and macroscopic properties of a system~\citep{hansen1969phase, bengtzelius1986dynamics, lin2003two, luo2004nonequilibrium}. 
The Lennard-Jones potential describes the interaction between two atoms $i$ and $j$ through the potential of the following form:
\begin{equation}
    V_{ij}(r) = \begin{cases} 
4\epsilon_{ij} [(\sigma_{ij}/r)^{12}-(\sigma_{ij}/r)^{6} ]
 & \text{if } r \leq r_{\text{cut}}, \\
0 & \text{if } r > r_{\text{cut}}.
\end{cases} 
\end{equation}
% \begin{wrapfigure}[17]{R}{0.45\textwidth}
% \centering
% \vspace{-.7cm}%
% \hspace{1cm}%
% \begin{tabular}{c}
% \setlength{\tabcolsep}{0pt}
% \includegraphics[width=0.43\textwidth]{figures/structure.pdf}
% \end{tabular}
% \caption{
%    Results on the Lennard-Jones system with 800 atoms and $N=4800$. Forces on 50 atoms are used to train $\lxp$ for all the latent model structures. Each model is trained with ten repeats.
% }
% \label{fig:LJ}
% \end{wrapfigure}

All the results in this experiment will be shown in reduced Lennard-Jones units. We consider a three-dimensional Lennard-Jones fluid with $N_{\text{atoms}}=800$ atoms of the same type. 
The microscopic state consists of the positions and velocities of the 800 atoms, thus the microscopic dimension is $N = 4800$. We simulate the Lennard-Jones system under NVE ensemble using  LAMMPS~\citep{LAMMPS}. 

We choose the instantaneous temperature ($T$) as our
macroscopic observables:
\begin{align}
    T = \frac{2}{3(N_{\text{atoms}} -1)} \times \sum_{i=1}^{N_{\text{atoms}}} \frac{m_iv_i^2}{2}
\end{align}
here $v_i$ is the velocity of the $i$-th atom and $m_i=1$. We find another 31 closure variables using the autoencoder, then the latent dimension is 32. In our experiment, we also adopt the trajectory distribution of microscopic state$\bx$ for $\mathcal{D}$. For data generation with partial forces, we choose $p=1/16$. For each $\bx$ we randomly choose 50 atoms for force computations.  

\paragraph{Results} All the models are trained with the same size of data. \cref{fig:LJ} shows the test error on 10 test trajectories. The test errors of using $\lxp(p=1/16)$ are relatively small ($\sim 10^{-3}$), which validates the accuracy of our model on the Lennard-Jones system. It is easy to observe from \cref{fig:LJ}
that for all the latent model structures, models trained with $\lxp$ can always outperform those trained with $\lz$. This validates $\lxp$ is \emph{robust over different latent model structures}.

% \begin{figure}[h]
% \centering
% \includegraphics[width=\linewidth, trim=9px 10px 0 0, clip]{figures/Allen_Cahn_tra.pdf}
% \caption{Results of the Allen-Cahn system.  \emph{Left:} Example of torus initial condition.  \emph{Middle:} True and predicted trajectories of free energy with initial condition $r_1=0.3, r_2=0.1$.  \emph{Right:} Mean squared error of the free energy between the true and predicted trajectories. }
% \label{fig:Allen_Cahn_tra}
% \end{figure}

\subsection{Robustness to Different Microscopic dynamics}
\begin{table}[t]
    \definecolor{h}{gray}{0.9}
    \caption{Summary of the results on each system. Results of the Predator-Prey and Lennard-Jones (small) system are taken from \cref{sec:4.1}, \cref{sec:4.2}. For each system, $\lz$ and $\lxp$ are trained with the same size of data. }
    \label{tab:all}
    \vspace{2mm}
	\centering
	\begin{adjustbox}{max width=\linewidth}
		\begin{tabular}{l c c c c c l l}
            \Xhline{3\arrayrulewidth}\bigstrut\bigstrut
              & Micro dim $N$ & Observables  & Latent dim $d$ & Partial labels $p$ & \multicolumn{1}{c}{$\lz$ }  & \multicolumn{1}{c}{\cellcolor{h}$\lxp$ } \\
             \Xhline{1\arrayrulewidth}\bigstrut\bigstrut
             Predator-Prey system & 100 & $\Bar{u}, \Bar{v}$ & 4 & 1/5 & 3.19 \scalebox{0.6}{$\pm$0.60} $\times 10^{-3}$   & \textbf{1.34 }\scalebox{0.6}{$\pm$0.16} $\mathbf{\times 10^{-3}}$ \\ 

                Allen-Cahn system & 40000 & free energy $\mathcal{E}(v)$ & 16 & 1/25 & 6.93 \scalebox{0.6}{$\pm$2.80} $\times 10^{-3}$  &  \textbf{3.98 }\scalebox{0.6}{$\pm$1.58} $\mathbf{\times 10^{-3}}$\\
        
                Lennard-Jones system (small)& 4800 &  temperature $T$ & 32 & 1/16 & 4.45 \scalebox{0.6}{$\pm$ 2.03} $\times 10^{-3}$ & \textbf{1.17} \scalebox{0.6}{$\pm$ 0.18}$\mathbf{\times 10^{-3}}$ \\
                % & 38400 & & & 1/128 \\
                % & 129600 & & & 1/432 \\
                Lennard-Jones system (large)& 307200 &temperature $T$ &32 & 1/1024  &  - & \textbf{4.96} \scalebox{0.6}{$\pm$0.56} $\mathbf{\times 10^{-3}}$\\
			\Xhline{3\arrayrulewidth}
		\end{tabular}
	\end{adjustbox}
\end{table}

We have already validated the accuracy and force computation efficiency of $\lxp$ on the Predatory-Prey system and the Lennard-Jones system, but their microscopic dimension is still not big enough. In this subsection, we focus on the robustness of our method to different systems including those with much larger microscopic dimension. We will consider two large systems: the Allen-Cahn system and a larger Lennard-Jones system with $51200$ atoms. 

\paragraph{Allen-Cahn System}  The Allen-Cahn equation is widely used to model the phase transition process in binary mixtures~\citep{allen1979microscopic, del2008toda,shen2010numerical, kim2021fast, yang2023fast}. 
We consider the 2-dimensional Allen-Cahn equation with zero Neumann boundary condition on a bounded domain:
\begin{equation}\begin{aligned}\label{eq:Allen_cahn}
    \partial_t v &= \nabla^2 v - \frac{1}{\epsilon^2}F^{\prime}(v)\ \text{on}\ \Omega=[0,1]\times [0,1] \\
    \partial_{\mathbf{n}} v &=0 \ \text{on} \ \partial \Omega, \\
\end{aligned}\end{equation}
where $v(-1\leq v\leq 1)$ denotes the difference of the concentration of the two phases. $F(v)$ is usually chosen to be the double potential taking the form of $F(v) = \frac{1}{4}(v^2-1)^2$. 
The Allen-Cahn equation is the $L^2$ gradient flow of the free energy functional $\mathcal{E}(v)\in\bbR$ in \cref{eq:energy}  ~\citep{Bartels2015}.
\begin{equation}\label{eq:energy}
    \mathcal{E}(v) = \int_{\mu} \left( \frac{1}{\epsilon^2} F(v) +\frac{1}{2} \norm{\nabla v}_2^2 \right) \ud x \ud y
 \end{equation}

 The free energy functional $\bz^{\ast}=\mathcal{E}(v) $  is a macroscopic observable of wide interest, hence we choose $\mathcal{E}(v)$ as our target macroscopic observable.

 The spatial domain is discretized into $200\times 200$ grids, then the dimension of the microscopic system is $N=40000$. We find another $31$ closure variables using the autoencoder hence the total dimension of the macroscopic system is 32, which is much smaller compared to the dimension of the microscopic system. We consider $\mathcal{D}$ to be the trajectory distribution of $\bx$. We choose $p=1/25$, each time the forces on 1600 grids are calculated for training $\lxp$. 

 \paragraph{Lennard-Jones System (large)} 
To further demonstrate the capacity of our method, 
we scale up the Lennard-Jones system in \cref{sec:4.2} to encompass $51200$ atoms, then $N=307200$. 
The size of the simulation box is increased from $10\times10\times10$ to $40\times40\times40$ to keep the density unchanged.   

\paragraph{Results} For a summary of the experiments and the results, we refer the reader to \cref{tab:all}. 
Note that for the Lennard-Jones system (large), we still use the forces on $50$ atoms for training, thus $p=1/1024$. 
For the training of $\lxp(p=1/1024)$, only 5000 configurations with partial forces are used due to memory limit, which is equivalent to $5000/1024\approx 5$ training data with forces on all the atoms.  Obviously, training data with size 5 is way too small, hence the results of $\lz$ are not reported for this system. 

From the results shown in \cref{tab:all}, one can observe that for a variety of problems, including those modeled by partial differential equations or molecular dynamics simulations, $\lxp$ can always outperform $\lz$. This shows the \emph{robustness of our method to a variety of systems}.
Moreover, the success of our method on the Lennard-Jones system (large) demonstrates the ability and efficiency of our method when scaled to very large systems. 

\begin{wrapfigure}[18]{R}{0.45\textwidth}
\centering
\vspace{-.7cm}%
\hspace{1cm}%
\begin{tabular}{c}
\setlength{\tabcolsep}{0pt}
\includegraphics[width=0.43\textwidth]{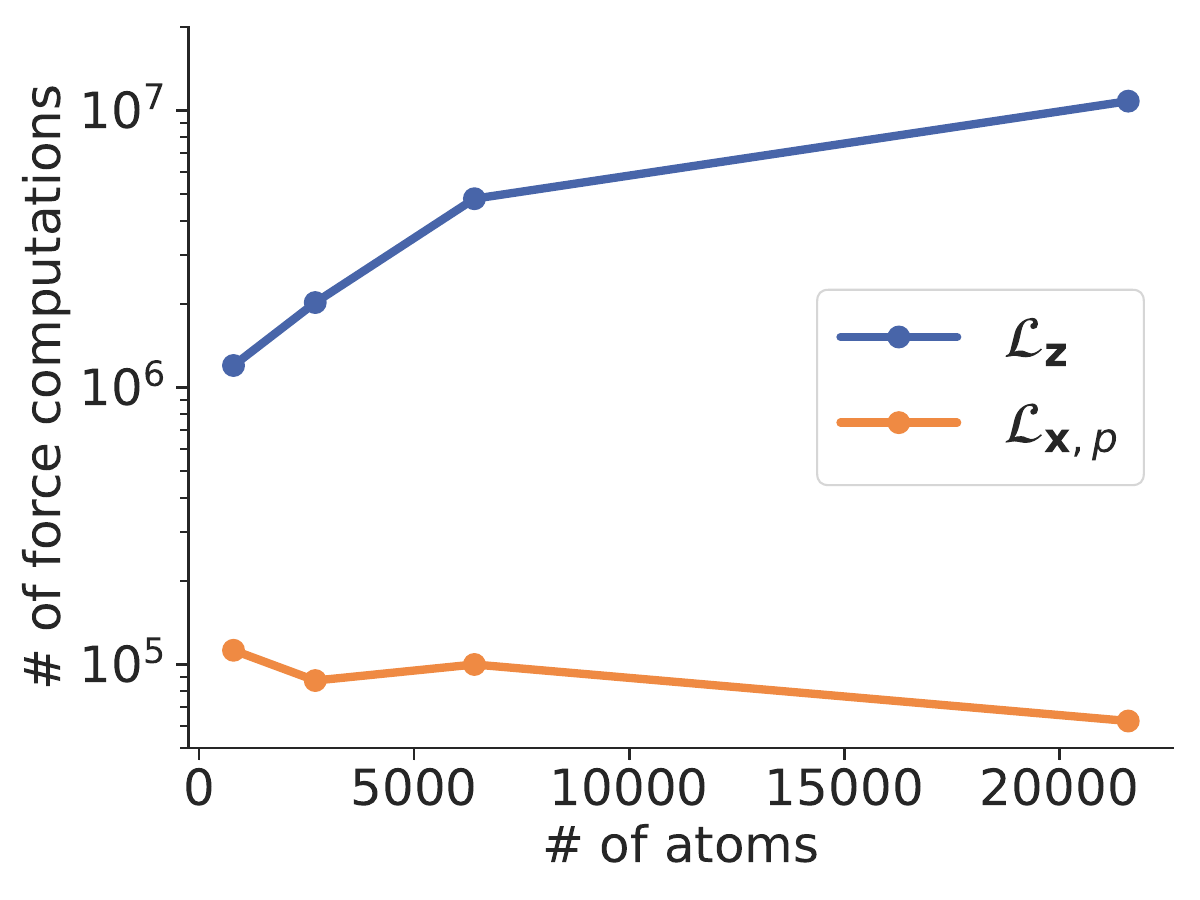}
\end{tabular}
\caption{
   Number of force computations required to achieve $e_{\text{tol}}=3\times10^{-3}$ on Lennard-Jones system with different sizes. Forces on 50 atoms are used to train $\lxp$ for systems of different sizes.
}
\label{fig:force}
\end{wrapfigure}
We also compare the number of force computations that are required for models trained with $\lz$ and $\lxp$ to reach test error $e_{\text{tol}}=3\times10^{-3}$. 
\cref{fig:force} shows the results on the Lennard-Jones system with different sizes. Lennard-Jones system with $800, 2700, 6400, 21600$ atoms are considered and the density is 0.8 for all the systems. The number of force computations here refers to the total number of forces on atoms that are used. For example, if $\lxp$ uses 100 configurations to train, and for each configuration, forces on 50 atoms are calculated, then the number of force computations would be $100\times50=5000$. 
From \cref{fig:force} we can observe that as the system size increases, the number of force computations required by $\lz$ continues increasing. 
In the experiment, we find that the number of training data does not change a lot. The increase in the number of force computations is mainly due to the system size increases, then for each configuration, more force computations are required. 
However, the number of force computations even decreases a bit. One possible explanation is that, as the system size increases, the dynamics become less fluctuating and are easy to learn.

% \begin{figure}[t]
%     \centering
%     \begin{tabular}{@{}c@{\hspace{5pt}}@{\hspace{5pt}}c@{\hspace{2pt}}c@{\hspace{2pt}}c@{}}
%          \includegraphics[width=0.24\textwidth]{figures/PP.pdf} & \includegraphics[width=0.24\textwidth]{figures/AC.pdf} &
%          \includegraphics[width=0.24\textwidth]{figures/LJ.pdf}\vspace{-5pt} & \includegraphics[width=0.24\textwidth]{figures/error.pdf}
%          \\
%          \quad \ \  {\scriptsize (a) Predator-Prey system }
%          & \ \ \
%          {\scriptsize (b) Allen-Cahn system}
%          &  \ \ 
%          {\scriptsize (c) Lennard-Jones system}
%          & \ \ 
%          {\scriptsize (d) Visualization of the results}\vspace{-5pt}
%     \end{tabular}
%     \caption{(a)\ (b)\ (c) Visualization of the microscopic state of each system (d) Visualization of the results in \cref{tab:all}. }
%     \label{fig:all}
%     \vspace{-10pt}
% \end{figure}

% \clearpage 

%% file: conclusion.tex
\section{Conclusion}\label{sec:conclusion}
We present a framework for modeling the dynamics of macroscopic observables from partial computation of microscopic forces. 
We theoretically and experimentally demonstrate the accuracy, force computation efficiency, and robustness of our method through different problems.  Finally, we apply our method to a very large Lennard-Jones system which contains $51200$ atoms. 

While our method can learn the macroscopic dynamics from partial computation of microscopic forces, it relies on the sparsity assumption. For systems that do not satisfy the sparsity assumption, the calculation of partial computation of microscopic forces is not efficient, thus it is not beneficial to learn from partial forces computation. For example, in the McKean-Vlasov system,  the force on each microscopic coordinate depends on the collective behavior of all the other coordinates~\citep{Méléard1996}.

Another limitation is the structure of the autoencoder.
For particle systems such as the Lennard-Jones system, an ideal encoder should be permutation-invariant. Currently, we use MLP for the encoder, which can be improved.
Additionally, our method assumes the microscopic state to be sampled from a distribution $\mathcal{D}$. We choose $\mathcal{D}$ to be trajectory distribution in the experiments, but in reality trajectory distribution of large systems may be impossible to obtain. Active learning is commonly applied to efficiently select microscopic configurations for training~\citep{ang2021active, zhang2019active, farache2022active, kulichenko2023uncertainty,  duschatko2024uncertainty}. 
It is of interest to combine active learning and our proposed method to overcome the difficulty of the choice of $\mathcal{D}$. Furthermore, our method assumes the microscopic dynamics to be deterministic, another future direction could be generalizing our method to stochastic systems. 

\subsubsection*{Acknowledgments}
This research is supported by the National Research Foundation, Singapore under its AI Singapore Programme (AISG Award No: AISG3-RP-2022-028).

% \clearpage 

%% file: appendix_0.tex
\newpage
\section*{Appendix}
Let $\bx=(\bx_1, \cdots, \bx_n) \in \mathcal{X} \subset \bbR^N$, $\by = \bfx\in \mathcal{Y}\subset \bbR^N, \btheta \in \Theta$, 
here $\Theta$ is the set of our model parameters. 
Write $\mathcal{D}^K = \{\bx^i, \by^i\}_{i=1,\cdots,K}$ ,
then $\lz(\btheta),\lx(\btheta)$ can be written as: 
\begin{equation}\begin{aligned}\label{eq:lzlx}
    \lz(\btheta) &= \frac{1}{K}\sum_{(\bx, \by)\in\mathcal{D}^K} \norm{\bphipx\by -\bg(\bz)}_2^2 \\ 
    \lx(\btheta) &=  \frac{1}{K}\sum_{(\bx, \by)\in\mathcal{D}^K}  \norm{\by  - (\bphipx)^{\dagger}\bg(\bz)}_2^2  \\
 %    \lxp(\btheta) &=    \frac{1}{K}\sum_{(\bx, \by)\in\mathcal{D}^K} 
 % \norm{\by_{\mathbf{I}(\mathbf{x})}  - (\bphipx)^{\dagger}_{\mathbf{I}(\mathbf{x})} \bg(\bz)}_2^2\\
\end{aligned}\end{equation}
In our method, we map the training procedure from the macroscopic coordinates to microscopic coordinates and use partial computation of the microscopic forces to train.    
We treat $\lz(\btheta)$ as baseline and give detailed theoretical analysis of the possible error introduced by using $\lxp(\btheta)$. 
The error can be controlled by two parts: (i) Convert loss from $\bz$ space to $\bx$ space, \ie, the error introduced by using $\lx(\btheta)$ (ii) Use partial labels, \ie, the error between $\lxp(\btheta)$ and $\lx(\btheta)$.

We will analyze the first part and the second part of the error accordingly~(\cref{Appendix:proof}). More experimental details are provided in \cref{Appendix:details}. 

%% file: appendix_proof.tex
\section{Theoretical Analysis}\label{Appendix:proof}
\subsection{Proof of Theorem 1}\label{Appendix:thm1-proof}
\begin{theorem*}
    Assume for any $\bx\sim \mathcal{D}$, the eigenvalues of $\bphipx\bphipx^T$ are lower bounded by $b_1$ and upper bounded by $b_2$, $0<b_1\leq b_2$. Then:
    \begin{align}
    b_1 (\lx(\btheta) + C) & \leq \lz(\btheta) \leq b_2 (\lx(\btheta) +C)
    \end{align}
    here C does not depend on $\btheta$ hence does not affect the optimization.  
\end{theorem*}
\begin{proof} 
We can write $\bphipx$ in the following form by leveraging singular value decomposition:  
\begin{align}
    \bphipx = \mathbf{U} \mathbf{\Sigma} \mathbf{V}^{T}, 
\end{align}
where $\bphipx\in \bbR^{d\times N},  \mathbf{\Sigma}\in \bbR^{d\times N}$ is a rectangular diagonal matrix,  $\mathbf{U}\in \bbR^{d\times d}$ and $ \mathbf{V}\in \bbR^{N \times N}$ are two orthogonal matrices, $ d\ll N$. $\mathbf{U}, \mathbf{\Sigma}, \mathbf{V}$ actually depends on $\bx$ but for simplicity we omit the dependence in notation. During the training of the autoencoder, we enforce $\bphipx$ to have full row rank, then the diagonal items are all nonzero, \ie, $\mathbf{\Sigma}_{i,i} = \lambda_i(\bx) \neq 0, i=1,\cdots, d$. 
We define $\mathbf{\Sigma}^{\dagger}\in \bbR^{N\times d}$ to be the rectangular diagonal matrix such that the diagonal items are $(\mathbf{\Sigma}^{\dagger})_{i, i} = \lambda_i^{-1}(\bx) \neq 0, i=1,\cdots, d$ and all the remaining items are zero. Actually $\mathbf{\Sigma}^{\dagger}$ is the Moore-Penrose inverse of $\mathbf{\Sigma}$. 

Then $(\bphipx)^{\dagger}$ can be calculated by:
\begin{align}
    (\bphipx)^{\dagger} = \mathbf{V} \mathbf{\Sigma}^{\dagger} \mathbf{U}^T  
\end{align}

If we denote the $i$-th column of $\mathbf{V}$ by $\mathbf{v}_i$, the i-th column of $\mathbf{U}$ by $\mathbf{u}_i$, then $\mathbf{V} = (\mathbf{v}_1, \cdots,\mathbf{v}_N), \mathbf{U} =(\mathbf{u}_1, \cdots, \mathbf{u}_d)$, and we can rewrite $\lz(\btheta)$ and $\lx(\btheta)$:
\begin{equation}
    \begin{aligned}
        \lz(\btheta) &= \textstyle \frac{1}{K} \sum_{(\bx,\by)\in\mathcal{D}^K} \norm{\bphipx\by  -\bg(\bz)}_2^2 \\
        &= \textstyle \frac{1}{K} \sum_{(\bx,\by)\in\mathcal{D}^K} 
 \norm{\mathbf{U} \mathbf{\Sigma} \mathbf{V}^{T} \by  -\bg(\bz)}_2^2  \\
        &=  \textstyle\frac{1}{K} \sum_{(\bx,\by)\in\mathcal{D}^K} \norm{\mathbf{U} \mathbf{\Sigma} \mathbf{V}^{T} \by  -\mathbf{U}\mathbf{U}^T\bg(\bz)}_2^2  \\ 
        &=\textstyle \frac{1}{K} \sum_{(\bx,\by)\in\mathcal{D}^K}   (\mathbf{\Sigma} \mathbf{V}^{T} \by  -\mathbf{U}^T\bg(\bz))^T \mathbf{U}^T\mathbf{U} (\mathbf{\Sigma} \mathbf{V}^{T} \by  -\mathbf{U}^T\bg(\bz)) \\
        &=  \textstyle\frac{1}{K} \sum_{(\bx,\by)\in\mathcal{D}^K} \norm{\mathbf{\Sigma} \mathbf{V}^{T} \by  -\mathbf{U}^T\bg(\bz)}_2^2  \\
        &=  \frac{1}{K} \textstyle\sum_{i=1}^d  \sum_{(\bx,\by)\in\mathcal{D}^K} \norm{\lambda_i(\bx) \mathbf{v}_i^T \by - \mathbf{u}_i^T \mathbf{g}_{\btheta}(\bz) }_2^2   \\
    \end{aligned}
\end{equation}
\begin{equation}
    \begin{aligned}
        \lx(\btheta) &= \textstyle \frac{1}{K} \sum_{(\bx,\by)\in\mathcal{D}^K} \norm{\by  - (\bphipx)^{\dagger}\bg(\bz)}_2^2  \\
        &= \textstyle \frac{1}{K} \sum_{(\bx,\by)\in\mathcal{D}^K} \norm{\by - \mathbf{V} \mathbf{\Sigma}^{\dagger} \mathbf{U}^T\bg(\bz)}_2^2  \\
        &= \textstyle \frac{1}{K} \sum_{(\bx,\by)\in\mathcal{D}^K} \norm{\mathbf{V}\mathbf{V}^T\by  - \mathbf{V} \mathbf{\Sigma}^{\dagger} \mathbf{U}^T\bg(\bz)}_2^2)   \\
        &= \textstyle \frac{1}{K} \sum_{(\bx,\by)\in\mathcal{D}^K}  (\mathbf{V}^T\by - \mathbf{\Sigma}^{\dagger} \mathbf{U}^T\bg(\bz))^T \mathbf{V}^T \mathbf{V} (\mathbf{V}^T\by - \mathbf{\Sigma}^{\dagger} \mathbf{U}^T\bg(\bz) \\
        &= \textstyle \frac{1}{K} \sum_{(\bx,\by)\in\mathcal{D}^K} \norm{\mathbf{V}^T\by- \mathbf{\Sigma}^{\dagger} \mathbf{U}^T\bg(\bz)}_2^2  \\
        &= \textstyle \frac{1}{K}  \sum_{i=1}^d  \sum_{(\bx,\by)\in\mathcal{D}^K}   \norm{ \mathbf{v}_i^T \by - \lambda_i^{-1}(\bx) \mathbf{u}_i^T \mathbf{g}_{\btheta}(\bz)}_2^2   + \frac{1}{K} \sum_{i=d+1}^N \sum_{(\bx,\by)\in\mathcal{D}^K}   \norm{\mathbf{v}_i^T\by}_2^2   \\
        &= \textstyle \frac{1}{K}   \sum_{i=1}^d \sum_{(\bx,\by)\in\mathcal{D}^K}   \lambda_i^{-2}(\bx) \norm{ \lambda_i(\bx) \mathbf{v}_i^T \by -  \mathbf{u}_i^T \mathbf{g}_{\btheta}(\bz)}_2^2   + \frac{1}{K} \sum_{i=d+1}^N  \sum_{(\bx,\by)\in\mathcal{D}^K}   \norm{\mathbf{v}_i^T \by}_2^2   \\
    \end{aligned}
\end{equation}
% Note that $\sum_{i=d+1}^N  \frac{1}{K} \sum_{(\bx,\by)\in\mathcal{D}^K}   \norm{\mathbf{v}_i^T \by }_2^2 $ is a constant that does not depend on $\btheta$. 
We define:  
\begin{equation}
    \begin{aligned}
        \label{eq:B.lxbar}
    \lxhat(\btheta) &= \textstyle\frac{1}{K} \sum_{i=1}^d   \sum_{(\bx,\by)\in\mathcal{D}^K}   \lambda_i^{-2}(\bx) \norm{ \lambda_i(\bx) \mathbf{v}_i^T \by-  \mathbf{u}_i^T \mathbf{g}_{\btheta}(\bz)}_2^2  \\
    C &= -\textstyle  \frac{1}{K} \sum_{i=d+1}^N  \sum_{(\bx,\by)\in\mathcal{D}^K}   \norm{\mathbf{v}_i^T \by}_2^2 \\ 
    \end{aligned}
\end{equation} 
then $\lx(\btheta) = \lxhat(\btheta) - C $, $C$ does not depend on $\btheta$ and: 
\begin{align}
    \underset{\btheta}{\min\ } \lx(\btheta) \iff  \underset{\btheta}{\min\ }  \lxhat(\btheta).
\end{align}
Comparing $\lz(\btheta)$ and $\lxhat(\btheta)$, we observe that the only difference between $\lz(\btheta)$ and $\lxhat(\btheta)$ is that for every term $\norm{ \lambda_i(\bx) \mathbf{v}_i^T \bfx -  \mathbf{u}_i^T \mathbf{g}_{\btheta}(\bz)}_2^2 $, there is a constant $\lambda_i^{-2}(\bx)$ multiplied to it. 
Hence $\lxhat(\btheta)$ is a weighted version of $\lz(\btheta)$.  Note that if the eigenvalues of $\bphipx$ are $\lambda_i(\bx), i=1,\cdots, d$, then the eigenvalues of $\bphipx\bphipx^T$ are $\lambda_i^2(\bx), i=1,\cdots, d$.  

Since the eigenvalues of $\bphipx\bphipx^T$ are lower bounded by $b_1>0$ and upper bounded by $b_2$, \ie, $\forall \bx \in \mathcal{X}, 0<b_1\leq \lambda_i^2(\bx) \leq b_2$, then:
\begin{align}
    b_2^{-1} \lz(\btheta) & \leq \lx(\btheta) + C\leq b_1^{-1} \lz(\btheta) 
\end{align}
or equivalently,
\begin{align}
    b_1 (\lx(\btheta) + C) & \leq \lz(\btheta) \leq b_2 (\lx(\btheta) + C)
\end{align}
 \end{proof}
By minimizing $\lxhat(\btheta)$, 
we are actually narrowing the region of $\lz(\btheta)$. Hence we want $b_1$ and $b_2$ to be as close as possible.  
In the extreme case where $b_1 = b_2$, minimizing $\lxhat(\btheta)$ is just equivalent to minimizing $\lz(\btheta)$. 
Another observation is that 
$\lxhat(\btheta)$ 
is a weighted version of $\lx(\btheta)$, 
and if there exists $i$ such that 
$\lambda_i(\bx)^2$ is too small compared to the others, 
 the weighted sum will be dominated by it
 in \cref{eq:B.lxbar}. 

 The above insights guide us to constrain the condition number of  $\bphipx\bphipx^T$ during the training of autoencoder, \ie, we require  $\bphipx\bphipx^T$ to be well-conditioned through ~\cref{eq:cond}. 
 % From the experiment in the Predator-Prey system, we found that using $\lx$ does not deteriorate the performance, but can enhance the stability and the performance slightly. 

\subsection{Proof of \cref{eq:equal}} \label{Appendix:B.2}
\begin{equation}
    \begin{aligned}
        \bbE_{\bx^1, \cdots, \bx^K} \bbE_{\mathbf{I}(\bx^1), \cdots, \mathbf{I}(\mathbf{x}^K)}\lxp(\btheta) &= \textstyle  \bbE \big[  \frac{1}{pK} \sum_{i=1}^K \norm{\mathbf{f}_{\mathbf{I}(\bx^i)}(\bx^i)  - (\bphi^{\prime}(\bx^i))^{\dagger}_{\mathbf{I}(\bx^i)} \bg(\bz^i)}_2^2 \big] \\
    &=  \textstyle \frac{1}{p}\bbE_{\bx}\bbE_{\I} \big[\norm{\mathbf{f}_{\I}(\bx)  - (\bphipx)^{\dagger}_{\I} \bg(\bz)}_2^2 \big] \\
    &=  \textstyle \frac{1}{p}\bbE_{\bx}\bbE_{\I} \big[ \sum_{i=1}^n \mathbf{I}_{i}(\bx) 
 \cdot \norm{\mathbf{f}_{i}(\bx)  - (\bphipx)^{\dagger}_{i} \bg(\bz)}_2^2  \big] \\
 &=  \textstyle \frac{1}{p}\bbE_{\bx}\big[ \sum_{i=1}^n \bbE_{\I}  \mathbf{I}_{i}(\bx) 
 \cdot \norm{\mathbf{f}_{i}(\bx)  - (\bphipx)^{\dagger}_{i} \bg(\bz)}_2^2  \big] \\
 &=  \textstyle\bbE_{\bx}\big[ \sum_{i=1}^n  \norm{\mathbf{f}_{i}(\bx)  - (\bphipx)^{\dagger}_{i} \bg(\bz)}_2^2  \big] \\
 &=  \textstyle\bbE_{\bx}\big[\norm{\mathbf{f}(\bx)  - (\bphipx)^{\dagger} \bg(\bz)}_2^2  \big] \\
 &=\bbE_{\bx^1, \cdots, \bx^K} \lx(\btheta)\\
    \end{aligned}
\end{equation}

\subsection{Proof of Theorem 2}
\label{Appendix:B.3}
% We prove in \cref{Appendix:B.2} that $ \bbE \lx = \bbE \lxp $. But we can only have finite training data in reality, so generally $\lx\neq\lxp$. 
In this section we will prove the behavior of the minimizer found by $\lxp$ in the limit. 
Our proof relies on the statistical learning theory and especially Rademacher complexity. We will provide some background information first. 

Let $\mathcal{H}$ be a family of real-valued functions with domain $\mathcal{W}$ and integrable \wrt $\mathbb{P}$, 
 here $\mathbb{P}$ is a probability over $\mathcal{W}$. $\mathcal{W}^n = (\mathbf{w}^1, \cdots, \mathbf{w}^n)$ is a collection of \iid samples from probability distribution $\mathcal{P}$ defined over $\mathcal{W}$. 
 
%  The expected risk $L(h)$ and empirical risk $L_n(h)$ are defined to be:
% \begin{equation}\begin{aligned}
%     L(h) &= \int h(\bw) \ud \mathbb{P}(\bw) = \bbE [h] \\ 
%     L_n(h) &= \frac{1}{n}\sum_{i=1}^n h(\bw^i). \\
% \end{aligned}
% \end{equation}
% Let $h^{\ast}=\argmin_h L(h) $ 

% be the best possible function that minimizes the expected risk, 
% $h^{\ast}_{\mathcal{H}}=\argmin_{h\in\mathcal{H}} R(h)$ be the best possible function within $\mathcal{H}$. 
% In real applications we deal with empirical risk and are finding $h_n =\argmin_{h\in\mathcal{H}} R_n(h)$ that minimizes the empirical risk. 
% Usually we can't find the exact $h_n$, 
% but rather we find an approximate solution $\Tilde{h}_n$. 
% The excess risk $\mathcal{E} = \bbE[R(\Tilde{h}_n) - R(h^{\ast})]$ can be decomposed into three parts, 
% the approximation error $\mathcal{E}_{\text{app}}$, the estimation error $\mathcal{E}_{\text{est}}$, 
% and the optimization error $\mathcal{E}_{\text{opt}}$~\citep{bottou2007tradeoffs}:
% \begin{equation}
%     \begin{aligned}
%         \mathcal{E} &= \bbE[R(h_{\mathcal{H}}^{\ast}) - R(h^{\ast})] + \bbE[R(h_n) - R(h_{\mathcal{H}}^{\ast})] + \bbE[R(\Tilde{h}_n) - R(h_n)] \\
%         &=\mathcal{E}_{\text{app}} +  \mathcal{E}_{\text{est}} + \mathcal{E}_{\text{opt}}
%     \end{aligned}
% \end{equation}
We will use the tool of Rademacher complexity:  
\begin{definition}
Let $\mathcal{H}, \mathcal{W}^n, \mathbb{P}$ be defined as before.  
The empirical Rademacher complexity of $\mathcal{H}$ with respect to $\mathcal{W}^n$ is defined as:
\begin{align}
    \mathcal{R}_{\mathcal{W}^n}(\mathcal{H}) = \bbE_{\boldsymbol{\sigma}} \left [\sup_{h\in\mathcal{H}} \frac{1}{n}\sum_{i=1}^n\sigma_i h(\bw^i)\right],
\end{align}
$\boldsymbol{\sigma}=(\sigma_1,\cdots,\sigma_n), \{\sigma_i\}_{i=1}^n$ are independent random variables uniformly chosen from $\{-1, 1\}$, with $P(\sigma_i = 1) = P(\sigma_i = -1) = 0.5$. Taking the expectation with respect to $\mathcal{W}^n$ yields the Rademacher complexity of the functional class $\mathcal{H}$:
\begin{align}
    \mathcal{R}_n(\mathcal{H}) = \bbE_{\mathcal{W}^n} \bbE_{\boldsymbol{\sigma}} \left [\sup_{h\in\mathcal{H}} \frac{1}{n}\sum_{i=1}^n\sigma_i h(\bw^i)\right].
\end{align}
\end{definition}
Then one can derive the generalization bound in terms of the Rademacher complexity~\citep{Wainwright_2019}:

\begin{theorem}\label{thm1}
Assume $\mathcal{H}$ is uniformly bounded by $b$ (\ie, $ \norm{f }_\infty \leq b $. Then for all \( n \geq 1 \) and \( \delta \geq 0 \), we have
\begin{align}
\underset{h\in\mathcal{H}}{\sup} \left|\frac{1}{n}\sum_{i=1}^n h(\bw^i) - \bbE[h] \right| \leq 2\mathcal{R}_n(\mathcal{H}_{\bx}) + \delta
\end{align}
with probability at least $ 1 - 2\exp\left( -\frac{n \delta^2}{8b^2} \right) $. Consequently, as long as $ \mathcal{R}_n(\mathcal{H}_{\bx}) = o(1) $, we have $   \frac{1}{n}\sum_{i=1}^n f(X_i) - \bbE[f] \overset{a.s.}{\longrightarrow}  0 , \forall f \in \mathcal{H}_{\bx}$.
\end{theorem}

Now in our problem, let $\mathcal{H}_{\bx, p}$ be the following one function class:
\begin{equation}
    \begin{aligned}
         \mathcal{H}_{\bx, p} &= \{h_{\btheta,p}: \mathcal{X}\times \mathcal{Y}\times\mathcal{I}\rightarrow \mathbb{R}; h_{\btheta, p}(\bx, \by, \I) = \textstyle\frac{1}{p}\norm{\by_{\I} - (\bphipx)_{\I}^{\dagger} \bg(\bz)}_2^2 , \btheta \in \Theta\} \\
    \end{aligned}
\end{equation}
Then
\begin{equation}\begin{aligned}
    \lxp(\btheta) &=  \frac{1}{K} \sum_{(\bx,\by)\in\mathcal{D}^K} h_{\btheta, p}(\bx,\by,\I) \\
\end{aligned}\end{equation}

Now we can prove \cref{thm:2}: 
\begin{theorem*}
    Let $ \tlx(\btheta) = \bbE \lx(\btheta),   \btheta^{\ast} \in \arg \min_{\btheta} \tlx(\btheta)$, $\btheta_{K, p} \in \arg \min_{\btheta}\lxp(\btheta)$, if $\mathcal{H}_{\bx, p}$ is uniformly bounded by $b_{\bx, p}$ and $\mathcal{R}_K(\mathcal{H}_{\bx, p})=o(1)$,  then: 
    \begin{align}\begin{aligned}
    \tlx(\btheta_{K, p}) &-  \tlx(\btheta^{\ast}) \overset{a.s.}{\longrightarrow} 0 \\
    \end{aligned}\end{align}
\end{theorem*}
\begin{proof}
We define $\tlxp(\btheta) = \bbE \lxp(\btheta)$, then $\tlx = \tlxp$ by \cref{Appendix:B.2}. 
Applying \cref{thm1}, we get
\begin{align}
    \underset{\btheta \in \Theta}{\sup} \left| \lxp(\btheta) -\tlxp(\btheta) \right| \leq 2 \mathcal{R}_K({\mathcal{H}_{\bx, p}}) + \delta
\end{align}
with probability at least $1 - 2 \exp (-\frac{K\delta^2}{8b_{\bx, p}})$, and $\lxp(\btheta) \overset{a.s.}{\longrightarrow}\tlxp(\btheta), \forall \btheta \in\Theta$. 

 Note that $\btheta^{\ast} \in \arg \min_{\btheta} \tlx(\btheta) \in \arg \min_{\btheta} \tlxp(\btheta) $, then 
    \begin{equation}\label{eq:bound_lx}
 \begin{aligned}
     0\leq \tlx(\btheta_{K, p}) - \tlx(\btheta^{\ast}) = \underbrace{\tlxp(\btheta_{K, p}) - \lxp(\btheta_{K, p})}_{\overset{a.s.}{\longrightarrow}0} &+ \underbrace{\lxp(\btheta_{K, p}) - \lxp(\btheta^{\ast})}_{\leq 0} \\ &+  \underbrace{\lxp(\btheta^{\ast}) - \tlxp(\btheta^{\ast})}_{\overset{a.s.}{\longrightarrow}0}
 \end{aligned}\end{equation} 
 thus 
$
    \tlxp(\btheta_{K, p}) - \tlxp(\btheta^{\ast}) \overset{a.s.}{\longrightarrow} 0 
$.
\end{proof}

%% file: appendix_details.tex
\section{Experiment Details}\label{Appendix:details}
\begin{figure}[h]
    \centering
    \begin{tabular}{@{}c@{\hspace{5pt}}@{\hspace{5pt}}c@{\hspace{2pt}}c@{}}
         \includegraphics[width=0.32\textwidth]{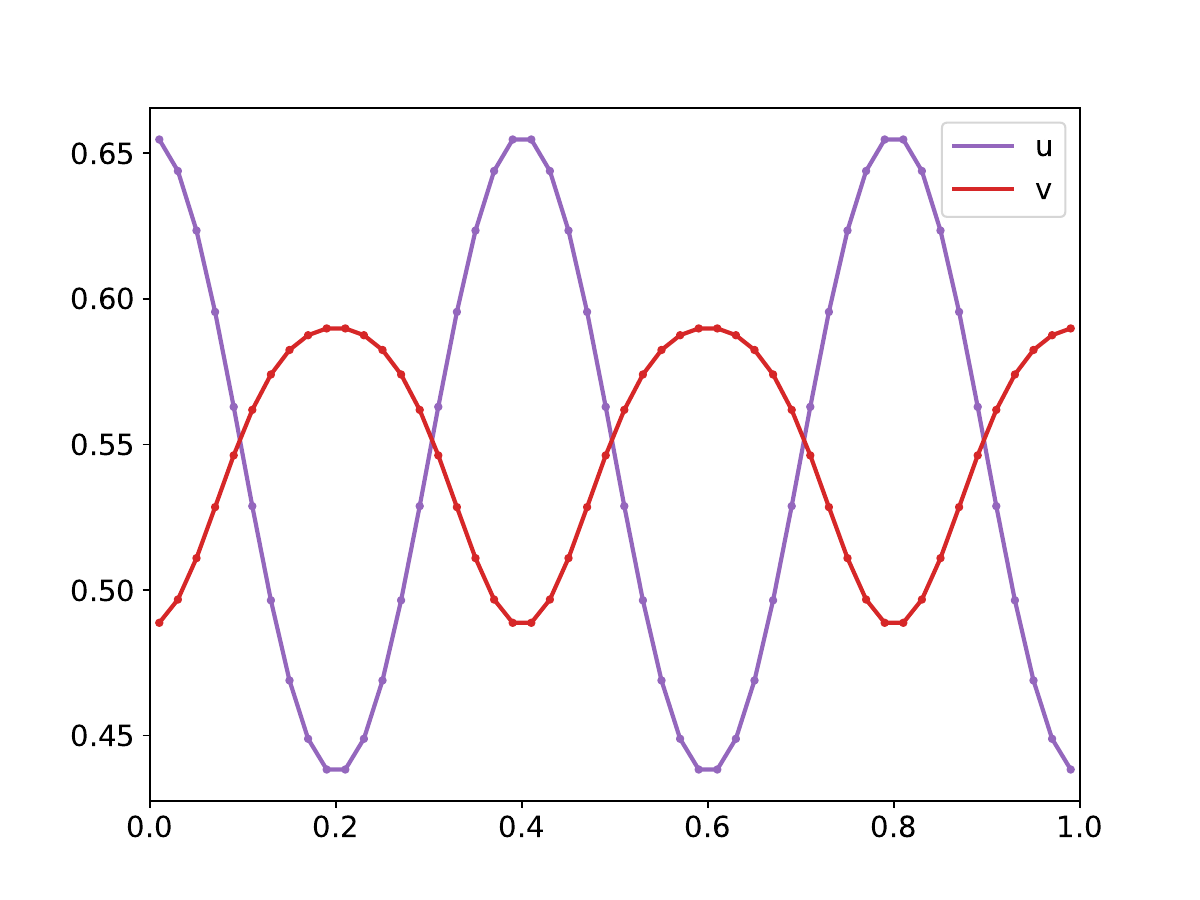} & \includegraphics[width=0.32\textwidth]{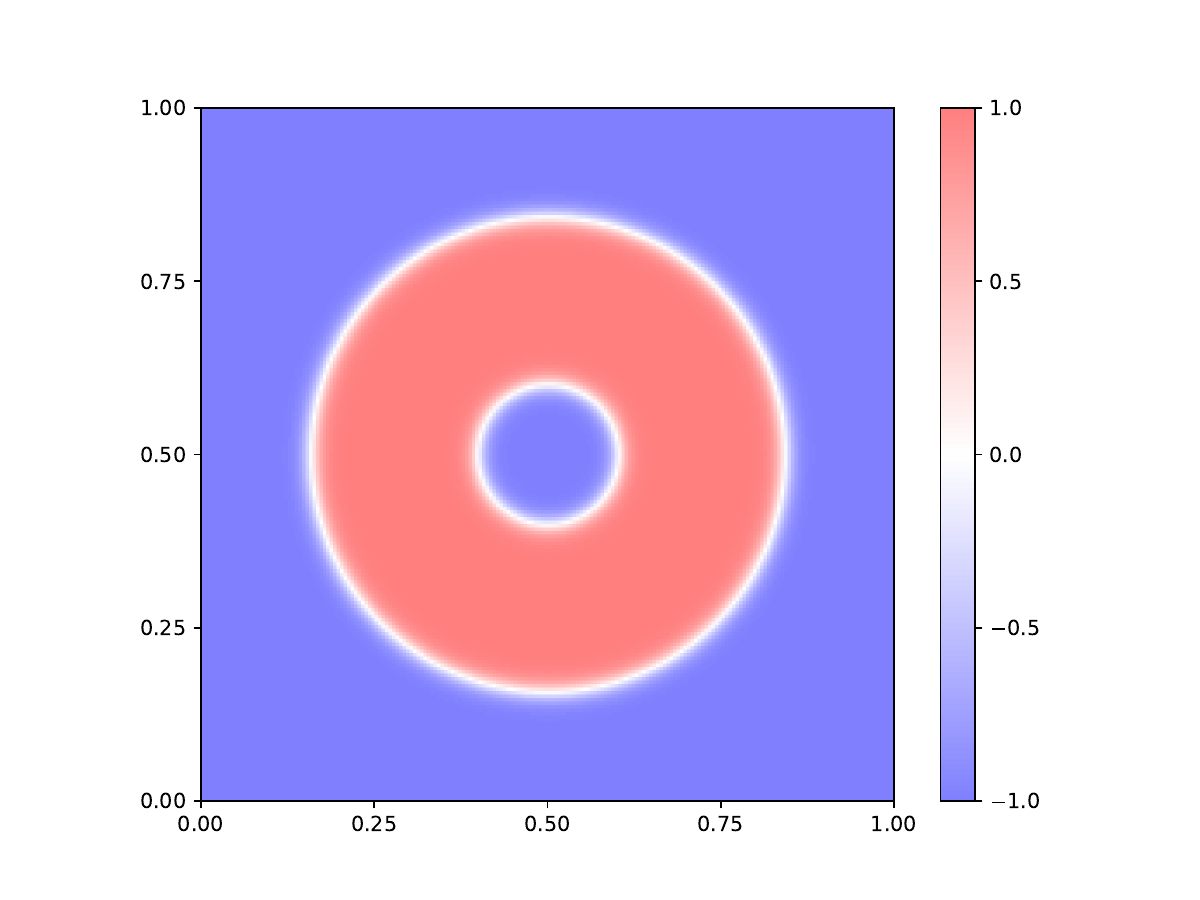} &
         \includegraphics[width=0.32\textwidth]{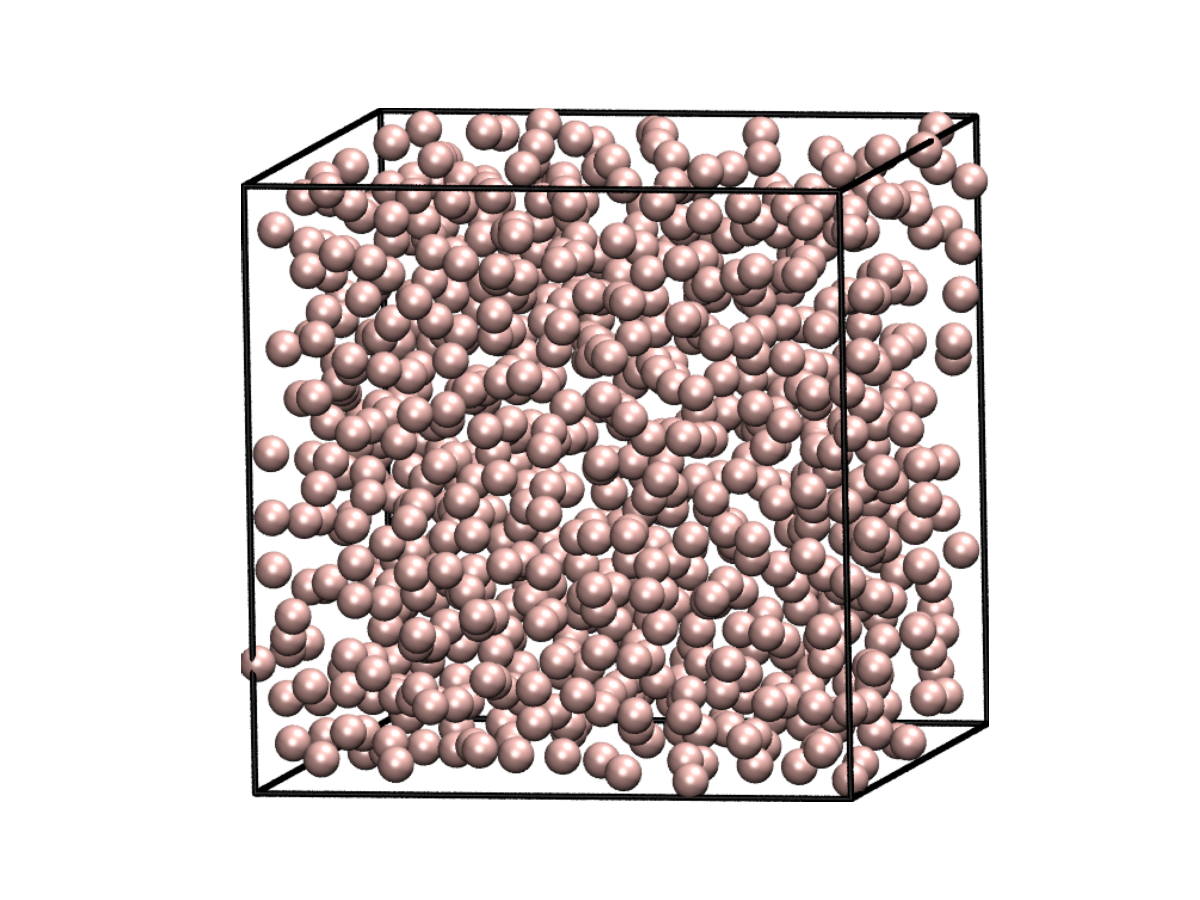}\vspace{-5pt}
         \\
         \quad \ \  {(a) Predator-Prey system }
         & \ \ \
         {(b) Allen-Cahn system}
         &  \ \ 
         {(c) Lennard-Jones system}
        % \vspace{-5pt}
    \end{tabular}
    \caption{Visualization of the microscopic state of each system}
    \label{fig:exp}
    % \vspace{-5pt}
\end{figure}

\subsection{Predator-Prey System}\label{Appendix:C.1}
We consider the Neumann boundary condition 
$\partial_{n}u=0, \partial_{n}v=0 \ \text{on}\ \partial\Omega$ and the following initial conditions
\begin{equation}\begin{aligned}
     u(x, 0) &=  \mu + \sigma\cos(5\pi x) \\
     v(x, 0) &=  1-\mu - \sigma\cos(5\pi x), \quad (\mu, \sigma)\in [0,0.2]\times [0.4, 0.6] \\
\end{aligned} 
\end{equation}
The Neumann boundary condition is commonly used in mathematical models of ecosystems, restricting any movement of the species out of the boundary. 

 We approximate the spatial derivatives in \cref{eq:Predator-Prey} with finite difference method:
\begin{equation}\begin{aligned}
    \frac{\partial^2 v}{\partial x^2}(x_i, t) &\approx  \frac{v(x_{i+1}, t) - 2v(x_i, t) + v(x_{i-1}, t) }{\Delta x^2}  \quad 2\leq i \leq 49 \\
    \frac{\partial^2 v}{\partial x^2}(x_1, t) &\approx \frac{v(x_{2}, t) - v(x_1, t)}{\Delta x^2}\\
    \frac{\partial^2 v}{\partial x^2}(x_{50}, t) &\approx \frac{v(x_{49}, t) - v(x_{50}, t)}{\Delta x^2} \\
    \end{aligned}.
\end{equation}
Let $h_u(u,v) = u(1-u-v), h_v(u, v) = av(u-b)$, and $h_u(\bu,\bv),h_v(\bu,\bv)$ denote the element-wise application of $h_u, h_v$ to each $u(x_i, t), v(x_i, t), 1\leq i\leq 50$. Then
 \begin{equation} \begin{aligned}\label{eq:finite_difference}
    \frac{\ud \mathbf{u}}{\ud t} &= h_u(\bu,\bv) \\
    \frac{\ud \mathbf{v}}{\ud t} &= h_v(\bu,\bv)+ \mathbf{A}\mathbf{v} \\
\end{aligned}\end{equation}
here $\mathbf{A}\in\bbR^{50\times 50}$ is a matrix defined according to ~\cref{eq:finite_difference}. Hence the predator-prey system after spatial discretization can be written in the form of ~\cref{eq:ODE}. 

In our experiment,  we choose $a=3, b=0.4, \lambda=0$.  The training parameter set $\mathcal{T}_{\text{train}}$ of pairs $(\mu, \sigma)$  are sampled uniformly from $[0,0.2]\times[0.4,0.6]$.
For testing, $\mathcal{T}_{\text{test}}$ is also sampled uniformly from $[0,0.2]\times[0.4,0.6]$ , but with a different random seed from  $\mathcal{T}_{\text{train}}$. The mean relative error is defined as :
\begin{equation}\label{eq:error}
e(\mathcal{T_{\text{test}}}) = \frac{1}{|\mathcal{T_{\text{test}}}|}\sum_{(\mu,\sigma)\in \mathcal{T_{\text{test}}}}\left(\frac{\sum_{n} \norm{\bz^{\ast}_{\text{true}}(t_n;\mu,\sigma) - \bz^{\ast}_{\text{pred}}(t_n;\mu,\sigma)}_2^2}{\sum_n\norm{\bz^{\ast}_{\text{true}}(t_n;\mu,\sigma)}_2^2 }\right) ,
\end{equation}
here we use $\bz(\cdot; \mu,\sigma)$ to denote the dependency of the solution on the initial condition.

\paragraph{Data Generation} 
The microscopic equation is solved with a uniform time step $\Delta t=0.01$ from $t=0$ to $t=30$ using the Euler method. We subsampled every tenth snapshot for training. 
% For the models trained with partial labels, we randomly sample $n_p$ entries out of the $100$ entries and use their forces for training. 
During testing, the microscopic evolution equation is solved with the same $\Delta t=0.01$ using Runge-Kutta 4-order~RK4 solver. Then we encode the microscopic trajectories to obtain the ground truth latent trajectories. 
The predicted latent trajectories are obtained by encoding the initial microscopic state first, then solved using RK4 solver with $\Delta t=0.1, \Delta=0.5$ on $[0,30]$. \cref{fig:predator_prey_tra_0.1} and \cref{fig:predator_prey_tra_0.5} show the true and predicted trajectories. 

\begin{figure}[h]
\centering
\includegraphics[width=\linewidth, trim=9px 10px 0 0, clip]{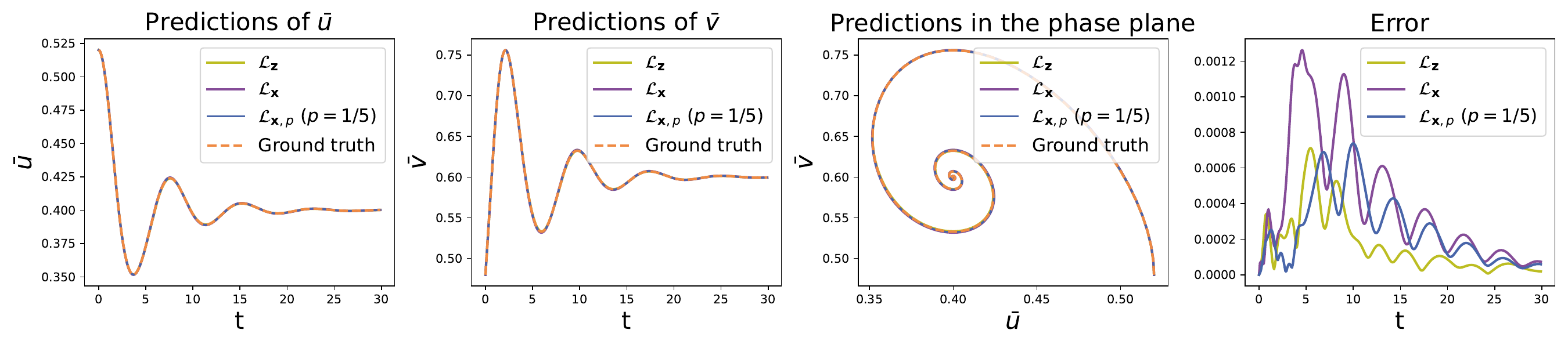}
\caption{Latent trajectories with initial condition $\mu=0.02, \sigma=0.52$ and $\Delta t=0.1$ in the Predator-Prey system.}
\label{fig:predator_prey_tra_0.1}
\end{figure}
%%%%%%%%%%%%%%%%%%%%%%%%%%%%%%%%%
\begin{figure}[h]
\centering
\includegraphics[width=\linewidth, trim=9px 10px 0 0, clip]{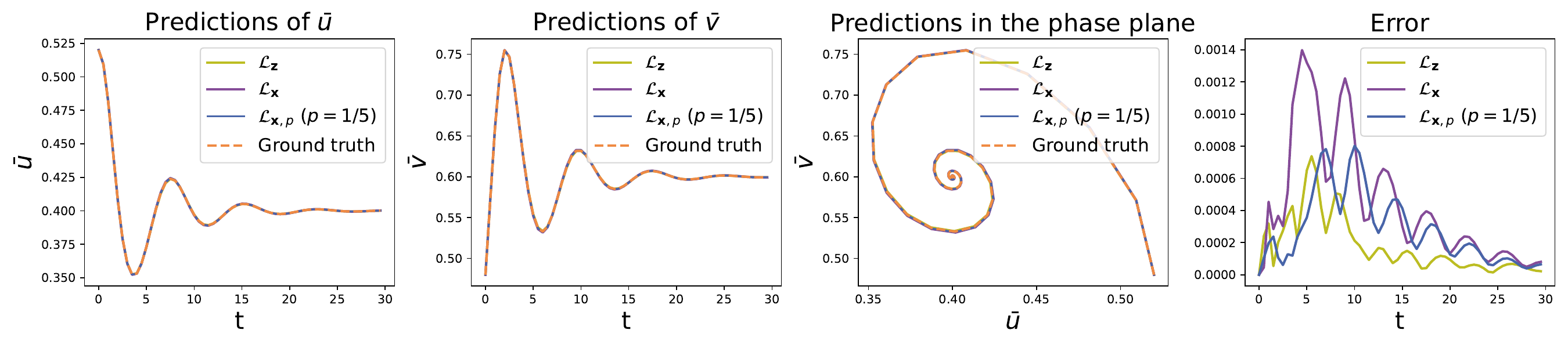}
\caption{Latent trajectories with initial condition $\mu=0.02, \sigma=0.52$ and $\Delta t=0.5$ in the Predator-Prey system.}
\label{fig:predator_prey_tra_0.5}
\end{figure}

\begin{table}[t]
    \definecolor{h}{gray}{0.9}
    \caption{Results on the Predator-Prey system. Models are trained with different training metric $\lx, \lxp(p=3/4,1/2,1/4,1/5)$. 
    Mean and standard deviation are reported over three repeats. } 
    \label{tab:P-P}
    \vspace{3mm}
	\centering
	\begin{adjustbox}{max width=\linewidth}
		\begin{tabular}{c c l l l l}
			\Xhline{3\arrayrulewidth}\bigstrut\bigstrut
			\# of training data & $\lx$ & \cellcolor{h} $\lxp(p=3/4)$& \cellcolor{h}$\lxp(p=1/2)$ & \cellcolor{h}$\lxp(p=1/4)$ & \cellcolor{h}$\lxp(p=1/5)$ \\
            %%%%%%%%%%%%%%%%%%%%%%%%%%%%%%%
			\Xhline{1\arrayrulewidth} \\[-1em]
                % \midrule
                $6.0 \times10^2$ & 
                1.09\scalebox{0.6}{ $\pm$ 0.57} $\times  10^{-2}$  & 
                1.79\scalebox{0.6}{ $\pm$ 1.11} $\times  10^{-2}$ & 
                9.22 \scalebox{0.6}{$\pm$ 4.46}$\times  10^{-3}$& 
                \textbf{3.50} \scalebox{0.6}{ $\pm$ 1.33} $\mathbf{\times  10^{-3}}$ & 
                4.00\scalebox{0.6}{ $\pm$ 1.78} $\times  10^{-3}$ \\
                \\[-1em]%%%%%%%%%%%%%%%%%%%%%%%%%%%%%%%%%%%%%%%%%%%%%%%
                $1.5\times10^3$ & 
                4.70\scalebox{0.6}{ $\pm$ 0.46} $\times  10^{-3}$  & 
                4.56\scalebox{0.6}{ $\pm$ 1.97}$\times  10^{-3}$ & 
                3.30\scalebox{0.6}{ $\pm$ 1.60} $\times  10^{-3}$ & 
                2.19 \scalebox{0.6}{ $\pm$ 0.57}$\times  10^{-3}$  & 
                \textbf{1.92}\scalebox{0.6}{ $\pm$ 0.47}$\mathbf{\times  10^{-3}}$  \\
                \\[-1em]%%%%%%%%%%%%%%%%%%%%%%%%%%%%%%%%%%%%%%%%%%%%%%%
                $3.0 \times10^3$ & 
                2.90\scalebox{0.6}{ $\pm$ 1.63} $\times  10^{-3}$&
                2.07\scalebox{0.6}{ $\pm$ 0.41} $\times  10^{-3}$ &
                1.57\scalebox{0.6}{ $\pm$ 0.45} $\times  10^{-3}$  & 
                \textbf{1.23}\scalebox{0.6}{ $\pm$ 0.16} $\mathbf{\times  10^{-3}}$   &
                1.34\scalebox{0.6}{ $\pm$ 0.20} $\times  10^{-3}$  \\
                \\[-1em]%%%%%%%%%%%%%%%%%%%%%%%%%%%%%%%%%%%%%%%%%%%%%%%
                $6.0 \times10^3$ &
                1.24\scalebox{0.6}{ $\pm$ 0.16}$\times  10^{-3}$ &
                1.07\scalebox{0.6}{ $\pm$ 0.18} $\times  10^{-3}$  &  
                1.46 \scalebox{0.6}{ $\pm$ 0.59} $\times  10^{-3}$ &  
                \textbf{9.08}\scalebox{0.6}{ $\pm$ 0.25} $\mathbf{\times  10^{-4}}$   &  1.13\scalebox{0.6}{ $\pm$ 0.31} $\times  10^{-3}$  \\
                \\[-1em]%%%%%%%%%%%%%%%%%%%%%%%%%%%%%%%%%%%%%%%%%%%%%%%
                $1.2\times10^4$ &   
                9.20\scalebox{0.6}{ $\pm$ 2.30} $\times  10^{-4}$ & 
                8.13\scalebox{0.6}{ $\pm$ 0.64} $\times  10^{-4}$  & 
                8.23\scalebox{0.6}{ $\pm$ 0.13} $\times  10^{-4}$  & 
                \textbf{7.66}\scalebox{0.6}{ $\pm$ 3.14} $\mathbf{\times  10^{-4}}$& 
                8.36\scalebox{0.6}{ $\pm$ 1.94} $\times  10^{-4}$ \\
                \\[-1em]%%%%%%%%%%%%%%%%%%%%%%%%%%%%%%%%%%%%%%%%%%%%%%%
                $2.4\times10^4$ & 
                6.76\scalebox{0.6}{ $\pm$ 0.93} $\times  10^{-4}$ & 
                6.85\scalebox{0.6}{ $\pm$ 0.95}$\times  10^{-4}$ & 
                6.98 \scalebox{0.6}{ $\pm$ 1.20}$\times  10^{-4}$  & 
                \textbf{5.64}\scalebox{0.6}{ $\pm$ 0.18} $\mathbf{\times  10^{-4}}$  & 
                5.70\scalebox{0.6}{ $\pm$ 1.00} $\times  10^{-4}$  \\
			\Xhline{3\arrayrulewidth}
		\end{tabular}
	\end{adjustbox}
\end{table}

\subsection{Allen-Cahn System}
In our experiment, we consider the initial condition of a torus ~\citep{kim2021fast}:
 \begin{equation}
 \begin{aligned}
 \label{eq:allan_cahn_initial}
     v(x,y,0) = -1 + \tanh(\frac{r_1 - d(x,y)} 
     {\sqrt{2}\epsilon}) - \tanh(\frac{r_2 - d(x,y)}
     {\sqrt{2}\epsilon})
 \end{aligned}
 \end{equation}
 here $d(x,y) = \sqrt{(x-0.5)^2+(y-0.5)^2}$, $r_1 \in [0.3, 0.4]$ is the circumscribed circle radius and $r_2\in [0.1,0.15]$ is the inscribed circle radius. The initial condition is visualized in \cref{fig:exp} (b).

The free energy in \cref{eq:energy} also tends to decrease with time, following the energy dissipation law in \cref{eq:energy_gradient}. Then the minimization of the free energy drives the evolution of the system towards equilibrium.
 \begin{equation}\label{eq:energy_gradient}
 \begin{aligned}
    \frac{\partial \mathcal{E}(v)}{\partial t} = -\int_{\mu} \norm{\partial_t v}_2^2  \ud x \ud y
 \end{aligned}\end{equation}

\paragraph{Data Generation}
For both training and testing, 
the microscopic evolution law is solved using RK4 method with $\Delta t = 1/N = 2.5\times 10^{-5}$ from $t=0$ to $t= \text{min}(t_f, 1)$.Here $t_f$ is the time when the Allen-Cahn system reaches the equilibrium. We subsample every hundredth snapshots for training. 
We choose $\epsilon$ in \cref{eq:Allen_cahn} and \cref{eq:allan_cahn_initial} to be $\frac{10\times 200}{2\sqrt{2}\tanh^{-1}(0.9)}$ as in ~\citep{kim2021fast}.

 For testing, 50 parameter points $(r_1, r_2)$ are chosen uniformly from $[0.3, 0.4] \times [0.1, 0.15]$, and we report the mean relative error. The test parameter set $\mathcal{T}_{\text{test}}$ contains 50 parameter points $(r_1, r_2)$ are chosen uniformly from $[0.3, 0.4] \times [0.1, 0.15]$, but are sampled with a different random seed. 

\subsection{Lennard-Jones System}
We consider Lennard-Jones systems containing different atoms in this paper: $N_{\text{atoms}}=800, 2700, 6400, 21600, 51200$. We use periodic boundary conditions and fix the density to be 0.8, then the corresponding box side lengths are $10, 15, 20, 30, 40$. 
We simulate the Lennard-Jones system under the NVE ensemble using the LAMMPS ~\citep{LAMMPS}. 
 In our experiment, $\epsilon_{ij} = \sigma_{ij} =1, \forall i,j $ , and $r_{\text{cut}}=2.5$. 
The integration step is 0.001 and each trajectory is integrated for 250 steps. We sample the initial temperature randomly from $[0.5, 1.5]$. The initial velocities are then sampled from the Maxwell–Boltzmann distribution. 
For each system, the initial configuration has the same atom positions and velocity direction.
For testing, the initial temperatures are also randomly sampled from $[0.5, 1.5]$ but with a different random seed to the training data. 

%% file: Appendix_D.tex
\section{Implementation Details}\label{Appendix:D}
All the experiments are run on a single NVIDIA GeForce RTX 3090 GPU. 
For all the experiments, we use the multilayer perceptron ~(MLP) for both the encoder and the decoder. 
% All the autoencoders are trained using the Adam optimizer with the parameters $\beta_1 =0.9,\beta_2=0.95$, wight decay = 1e-4, and learning rate = 1e-4. 
% \cref{tab:ae} provides an overview of the hyperparameters of all the autoencoders.
The autoencoders are trained with $\mathcal{L}_{\text{AE}}$ in  \cref{eq:AE_loss}. The condition number is the maximal eigenvalue $\lambda_{\max}$ divided by  the minimal eigenvalue $\lambda_{\min}$ of $\bphipx \bphipx^T$: 
\begin{align}
    \kappa\left(\bphipx\bphipx^T \right) &= \frac{|\lambda_{\max}(\bphipx \bphipx^T)|}{|\lambda_{\min}(\bphipx \bphipx^T)|} \geq 1 
\end{align}
Since $\bphipx \in \bbR^{d\times N}, \bphipx \bphipx^T\in \bbR^{d\times d}$, and $d$ is small, the condition number $\kappa(\bphipx \bphipx^T)$ can be calculated efficiently. In our experiments, we calculate $\lambda_{\text{max}}$ and $\lambda_{\text{min}}$ with \texttt{torch.linalg.svd}. 
To better compare $\lx$ and $\lxp$, once finish the training of the autoencoders, we freeze them and use the encoder for macroscopic dynamics identification. For the macroscopic dynamics identification,  MLP and GFINNs are used for the latent model in \cref{sec:4.2}. For the rest experiments, we adopt the structure of OnsagerNet to enhance the stability for latent dynamics prediction for $\bg$~\citep{yu2021onsagernet}. 

% \begin{table}[h]
% \centering
% \caption{Hyperparameters of the autoencoders used in each experiment. The same number of hidden dimensions and layers are used for both the encoder and the decoder.}\label{tab:ae}
% \begin{tabular}{l c c c}
% \toprule
%  & Predator-Prey system & Allen-Cahn system& Lennard-Jones system \\
% \midrule
% \# of layers & 2 & 2 & 2 \\
% Hidden Dimension & 32 & 64 & 64 \\
% Activation Function & Softplus & Softplus & Softplus \\
% Training Batch size & 256 & 256 & 256 \\
% Epochs & 10 & 20 & 200 \\
% Learning Rate & 1e-4 &  1e-4 & 2e-4 \\
% % Learning Rate Scheduler & - & - & - \\
% % Scheduler Weight Decay & - & - & - \\
% $\lambda_{\text{cond}}$ & 1e-6 & 1e-5& 1e-5 \\ 
% \bottomrule
% \end{tabular}
% \label{tab:hyper-params}
% \end{table}

% \subsection{Macroscopic Dynamics Identification}\label{Appendix:D.2}
% We adopt the structure of OnsagerNet to enhance the stability for latent dynamics prediction~(\citet{yu2021onsagernet}) for $\bg$. For the experiment in \cref{sec:4.2}

\section{Additional Experiments}

\subsection{Loss Curve of $\lz$, $\lx$, $\lxp$}

To give the readers a better idea of the behaviors of the loss $\lz, \lx, \lxp (p=1/4)$ trained on the same number of training data. \cref{fig:loss_plot} shows the training and test loss curve of different training metrics. Note that the training metrics of these models are different, but they are tested with the same metric \cref{eq:error}. 

\begin{figure}[t]
\centering
\includegraphics[width=\linewidth, trim=9px 10px 0 0, clip]{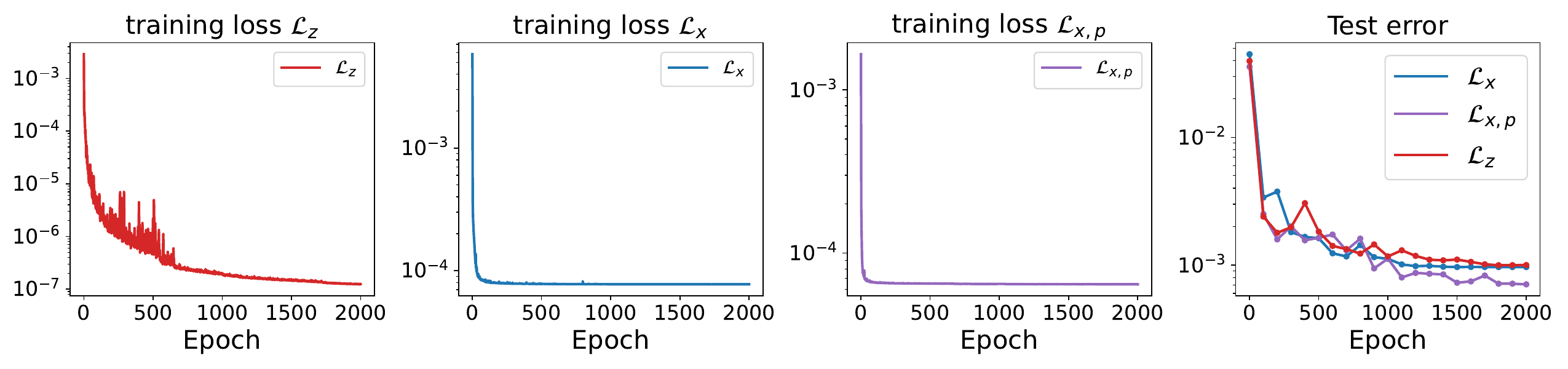}
\caption{
Loss curve of the $\lz, \lx, \lxp (p=1/4)$ on the Predator-Prey system. Models are trained with different loss functions $\lz, \lx, \lxp (p=1/4)$ on the same number of training data.
}
\label{fig:loss_plot}
\end{figure}

\subsection{Ablation Analysis of $\lambda_{\text{cond}}$}
To evaluate the influence of the hyperparameter $\lambda_{\text{cond}}$ on the performance of the loss $\lxp$, we conduct experiments with different values of $\lambda_{\text{cond}}$  and show the test error in \cref{tab:ablation}. 

\begin{table}[t]
    \definecolor{h}{gray}{0.9}
    \caption{Results on the Predator-Prey system. We train the autoencoder with different $\lambda_{\text{cond}}$ and then train the macroscopic dynamics model with loss $\lxp (p=1/5)$. The mean relative error of the macroscopic dynamics model is reported over three repeats. } 
    \label{tab:ablation}
    \vspace{3mm}
	\centering
	\begin{adjustbox}{max width=\linewidth}
		\begin{tabular}{c l l l l l l l }
			\Xhline{3\arrayrulewidth}\bigstrut
   $\lambda_{\text{cond}}$ &  $0$ & $10^{-8}$  &  $10^{-7}$ & $10^{-6}$ & $10^{-5}$ & $10^{-4}$ & $10^{-2}$   \\
            %%%%%%%%%%%%%%%%%%%%%%%%%%%%%%%
		\Xhline{1\arrayrulewidth} \\
   [-1em]%%%%%%%%%%%%%%%%%%%%%%%%%%%%%%%%%%%%%%%%%%%%%%%
               
            Test error of $\lxp (p= 1/5) $ &   $6.42\times 10^{-2} $ &	$2.62\times 10^{-3}$	&$2.84\times 10^{-3}$&	$\mathbf{8.36\times 10^{-4}}$	&$3.66\times 10^{-3}$	&$2.60\times 10^{-3}$&	$4.24\times 10^{-3}$  \\
			\Xhline{3\arrayrulewidth}
		\end{tabular}
	\end{adjustbox}
\end{table}
From \cref{tab:ablation} we can observe when $\lambda_{\text{cond}}$ increases from $0$ to $10^{-6}$, the test error gradually decrease. When $\lambda_{\text{cond}}$ further increases from $10^{-6}$ to $10^{-2}$, the test error gradually decrease. Among all the $\lambda_{\text{cond}}$ that we tried, the test error has the minimal value when $\lambda_{\text{cond}}=10^{-6}$. 

Theoretically, if $\lambda_{\text{cond}}$ is too low,  since $\mathcal{L}_{\text{AE}} =  \mathcal{L}_{\text{rec}} + \lambda_{\text{cond}} \mathcal{L}_{\text{cond}}$,
there may not have enough constraint on $\mathcal{L}_{\text{cond}}$ and 
the condition number of $\bphipx\bphipx^T$ may be very large.
By \cref{thm:1}, a small condition number of $\bphipx\bphipx^T$ can guarantee the effectiveness of $\lx$. But when the condition number is large, there is no guarantee that $\lx$ thus $\lxp$ can perform well. 
If instead $ \lambda_{\text{cond}} $ is too large, $\mathcal{L}_{\text{AE}} $ will be dominated by  $\mathcal{L}_{\text{cond}}$.Then the autoencoder may not reconstruct the microscopic dynamics well and, hence may not capture the closure terms well. If the latent space is not closed enough, we can not learn the macroscopic dynamics well.

\clearpage 

% lbd_list = [0, 1e-8, 1e-7, 1e-6, 1e-5, 1e-4, 1e-2]

%% file: main.bbl
\begin{thebibliography}{55}
\providecommand{\natexlab}[1]{#1}
\providecommand{\url}[1]{\texttt{#1}}
\expandafter\ifx\csname urlstyle\endcsname\relax
  \providecommand{\doi}[1]{doi: #1}\else
  \providecommand{\doi}{doi: \begingroup \urlstyle{rm}\Url}\fi

\bibitem[Allen et~al.(2004)]{allen2004introduction}
Michael~P Allen et~al.
\newblock Introduction to molecular dynamics simulation.
\newblock \emph{Computational soft matter: from synthetic polymers to proteins}, 23\penalty0 (1):\penalty0 1--28, 2004.

\bibitem[Allen \& Cahn(1979)Allen and Cahn]{allen1979microscopic}
Samuel~M Allen and John~W Cahn.
\newblock A microscopic theory for antiphase boundary motion and its application to antiphase domain coarsening.
\newblock \emph{Acta metallurgica}, 27\penalty0 (6):\penalty0 1085--1095, 1979.

\bibitem[Ang et~al.(2021)Ang, Wang, Schwalbe-Koda, Axelrod, and G{\'o}mez-Bombarelli]{ang2021active}
Shi~Jun Ang, Wujie Wang, Daniel Schwalbe-Koda, Simon Axelrod, and Rafael G{\'o}mez-Bombarelli.
\newblock Active learning accelerates ab initio molecular dynamics on reactive energy surfaces.
\newblock \emph{Chem}, 7\penalty0 (3):\penalty0 738--751, 2021.

\bibitem[Ayed et~al.(2019)Ayed, de~B{\'e}zenac, Pajot, Brajard, and Gallinari]{ayed2019learning}
Ibrahim Ayed, Emmanuel de~B{\'e}zenac, Arthur Pajot, Julien Brajard, and Patrick Gallinari.
\newblock Learning dynamical systems from partial observations.
\newblock \emph{arXiv preprint arXiv:1902.11136}, 2019.

\bibitem[Bakarji et~al.(2022)Bakarji, Champion, Kutz, and Brunton]{bakarji2022discovering}
Joseph Bakarji, Kathleen Champion, J~Nathan Kutz, and Steven~L Brunton.
\newblock Discovering governing equations from partial measurements with deep delay autoencoders.
\newblock \emph{arXiv preprint arXiv:2201.05136}, 2022.

\bibitem[Bartels(2015)]{Bartels2015}
S{\"o}ren Bartels.
\newblock \emph{The Allen--Cahn Equation}, pp.\  153--182.
\newblock Springer International Publishing, Cham, 2015.

\bibitem[Bengtzelius(1986)]{bengtzelius1986dynamics}
U~Bengtzelius.
\newblock Dynamics of a lennard-jones system close to the glass transition.
\newblock \emph{Physical Review A}, 34\penalty0 (6):\penalty0 5059, 1986.

\bibitem[Champion et~al.(2019)Champion, Lusch, Kutz, and Brunton]{champion2019data}
Kathleen Champion, Bethany Lusch, J~Nathan Kutz, and Steven~L Brunton.
\newblock Data-driven discovery of coordinates and governing equations.
\newblock \emph{Proceedings of the National Academy of Sciences}, 116\penalty0 (45):\penalty0 22445--22451, 2019.

\bibitem[Chen et~al.(2024)Chen, Soh, Ooi, Vissol-Gaudin, Yu, Novoselov, Hippalgaonkar, and Li]{chen2024constructing}
Xiaoli Chen, Beatrice~W Soh, Zi-En Ooi, Eleonore Vissol-Gaudin, Haijun Yu, Kostya~S Novoselov, Kedar Hippalgaonkar, and Qianxiao Li.
\newblock Constructing custom thermodynamics using deep learning.
\newblock \emph{Nature Computational Science}, 4\penalty0 (1):\penalty0 66--85, 2024.

\bibitem[Dajnowicz et~al.(2022)Dajnowicz, Agarwal, Stevenson, Jacobson, Ramezanghorbani, Leswing, Friesner, Halls, and Abel]{dajnowicz2022high}
Steven Dajnowicz, Garvit Agarwal, James~M Stevenson, Leif~D Jacobson, Farhad Ramezanghorbani, Karl Leswing, Richard~A Friesner, Mathew~D Halls, and Robert Abel.
\newblock High-dimensional neural network potential for liquid electrolyte simulations.
\newblock \emph{The Journal of Physical Chemistry B}, 126\penalty0 (33):\penalty0 6271--6280, 2022.

\bibitem[Del~Pino et~al.(2008)Del~Pino, Kowalczyk, and Wei]{del2008toda}
Manuel Del~Pino, Micha{\l} Kowalczyk, and Juncheng Wei.
\newblock The toda system and clustering interfaces in the allen--cahn equation.
\newblock \emph{Archive for rational mechanics and analysis}, 190\penalty0 (1):\penalty0 141--187, 2008.

\bibitem[Duschatko et~al.(2024)Duschatko, Vandermause, Molinari, and Kozinsky]{duschatko2024uncertainty}
Blake~R Duschatko, Jonathan Vandermause, Nicola Molinari, and Boris Kozinsky.
\newblock Uncertainty driven active learning of coarse grained free energy models.
\newblock \emph{npj Computational Materials}, 10\penalty0 (1):\penalty0 9, 2024.

\bibitem[Farache et~al.(2022)Farache, Verduzco, McClure, Desai, and Strachan]{farache2022active}
David~E Farache, Juan~C Verduzco, Zachary~D McClure, Saaketh Desai, and Alejandro Strachan.
\newblock Active learning and molecular dynamics simulations to find high melting temperature alloys.
\newblock \emph{Computational Materials Science}, 209:\penalty0 111386, 2022.

\bibitem[Fresca et~al.(2020)Fresca, Manzoni, Ded{\`e}, and Quarteroni]{fresca2020deep}
Stefania Fresca, Andrea Manzoni, Luca Ded{\`e}, and Alfio Quarteroni.
\newblock Deep learning-based reduced order models in cardiac electrophysiology.
\newblock \emph{PloS one}, 15\penalty0 (10):\penalty0 e0239416, 2020.

\bibitem[Fries et~al.(2022)Fries, He, and Choi]{fries2022lasdi}
William~D Fries, Xiaolong He, and Youngsoo Choi.
\newblock Lasdi: Parametric latent space dynamics identification.
\newblock \emph{Computer Methods in Applied Mechanics and Engineering}, 399:\penalty0 115436, 2022.

\bibitem[Fu et~al.(2023)Fu, Xie, Rebello, Olsen, and Jaakkola]{fu2023simulate}
Xiang Fu, Tian Xie, Nathan~J Rebello, Bradley Olsen, and Tommi~S Jaakkola.
\newblock Simulate time-integrated coarse-grained molecular dynamics with multi-scale graph networks.
\newblock \emph{Transactions on Machine Learning Research}, 2023.

\bibitem[Hafner et~al.(2006)Hafner, Wolverton, and Ceder]{hafner2006toward}
J{\"u}rgen Hafner, Christopher Wolverton, and Gerbrand Ceder.
\newblock Toward computational materials design: the impact of density functional theory on materials research.
\newblock \emph{MRS bulletin}, 31\penalty0 (9):\penalty0 659--668, 2006.

\bibitem[Hansen \& Verlet(1969)Hansen and Verlet]{hansen1969phase}
Jean-Pierre Hansen and Loup Verlet.
\newblock Phase transitions of the lennard-jones system.
\newblock \emph{physical Review}, 184\penalty0 (1):\penalty0 151, 1969.

\bibitem[Hernandez et~al.(2021)Hernandez, Badias, Gonzalez, Chinesta, and Cueto]{hernandez2021deep}
Quercus Hernandez, Alberto Badias, David Gonzalez, Francisco Chinesta, and Elias Cueto.
\newblock Deep learning of thermodynamics-aware reduced-order models from data.
\newblock \emph{Computer Methods in Applied Mechanics and Engineering}, 379:\penalty0 113763, 2021.

\bibitem[Huang et~al.(2020)Huang, Sun, and Wang]{huang2020learning}
Zijie Huang, Yizhou Sun, and Wei Wang.
\newblock Learning continuous system dynamics from irregularly-sampled partial observations.
\newblock \emph{Advances in Neural Information Processing Systems}, 33:\penalty0 16177--16187, 2020.

\bibitem[Husic et~al.(2020)Husic, Charron, Lemm, Wang, P{\'e}rez, Majewski, Kr{\"a}mer, Chen, Olsson, de~Fabritiis, et~al.]{husic2020coarse}
Brooke~E Husic, Nicholas~E Charron, Dominik Lemm, Jiang Wang, Adri{\`a} P{\'e}rez, Maciej Majewski, Andreas Kr{\"a}mer, Yaoyi Chen, Simon Olsson, Gianni de~Fabritiis, et~al.
\newblock Coarse graining molecular dynamics with graph neural networks.
\newblock \emph{The Journal of chemical physics}, 153\penalty0 (19), 2020.

\bibitem[Jia et~al.(2020)Jia, Wang, Chen, Lu, Lin, Car, E, and Zhang]{jia2020pushing}
Weile Jia, Han Wang, Mohan Chen, Denghui Lu, Lin Lin, Roberto Car, Weinan E, and Linfeng Zhang.
\newblock Pushing the limit of molecular dynamics with ab initio accuracy to 100 million atoms with machine learning.
\newblock In \emph{SC20: International conference for high performance computing, networking, storage and analysis}, pp.\  1--14. IEEE, 2020.

\bibitem[Kevrekidis et~al.(2003)Kevrekidis, Gear, Hyman, Kevrekidis, Runborg, Theodoropoulos, et~al.]{kevrekidis2003equation}
Ioannis~G Kevrekidis, C~William Gear, James~M Hyman, Panagiotis~G Kevrekidis, Olof Runborg, Constantinos Theodoropoulos, et~al.
\newblock Equation-free, coarse-grained multiscale computation: enabling microscopic simulators to perform system-level analysis.
\newblock \emph{Commun. Math. Sci}, 1\penalty0 (4):\penalty0 715--762, 2003.

\bibitem[Kim et~al.(2021)Kim, Ryu, and Choi]{kim2021fast}
Yongho Kim, Gilnam Ryu, and Yongho Choi.
\newblock Fast and accurate numerical solution of allen--cahn equation.
\newblock \emph{Mathematical Problems in Engineering}, 2021:\penalty0 1--12, 2021.

\bibitem[Kulichenko et~al.(2023)Kulichenko, Barros, Lubbers, Li, Messerly, Tretiak, Smith, and Nebgen]{kulichenko2023uncertainty}
Maksim Kulichenko, Kipton Barros, Nicholas Lubbers, Ying~Wai Li, Richard Messerly, Sergei Tretiak, Justin~S Smith, and Benjamin Nebgen.
\newblock Uncertainty-driven dynamics for active learning of interatomic potentials.
\newblock \emph{Nature Computational Science}, 3\penalty0 (3):\penalty0 230--239, 2023.

\bibitem[Lee \& Carlberg(2020)Lee and Carlberg]{lee2020model}
Kookjin Lee and Kevin~T Carlberg.
\newblock Model reduction of dynamical systems on nonlinear manifolds using deep convolutional autoencoders.
\newblock \emph{Journal of Computational Physics}, 404:\penalty0 108973, 2020.

\bibitem[Lee et~al.(2020)Lee, Kooshkbaghi, Spiliotis, Siettos, and Kevrekidis]{lee2020coarse}
Seungjoon Lee, Mahdi Kooshkbaghi, Konstantinos Spiliotis, Constantinos~I Siettos, and Ioannis~G Kevrekidis.
\newblock Coarse-scale pdes from fine-scale observations via machine learning.
\newblock \emph{Chaos: An Interdisciplinary Journal of Nonlinear Science}, 30\penalty0 (1), 2020.

\bibitem[Lin et~al.(2003)Lin, Blanco, and Goddard~III]{lin2003two}
Shiang-Tai Lin, Mario Blanco, and William~A Goddard~III.
\newblock The two-phase model for calculating thermodynamic properties of liquids from molecular dynamics: Validation for the phase diagram of lennard-jones fluids.
\newblock \emph{The Journal of chemical physics}, 119\penalty0 (22):\penalty0 11792--11805, 2003.

\bibitem[Liu et~al.(2015)Liu, Samaey, Gear, and Kevrekidis]{liu2015acceleration}
Ping Liu, Giovanni Samaey, C~William Gear, and Ioannis~G Kevrekidis.
\newblock On the acceleration of spatially distributed agent-based computations: A patch dynamics scheme.
\newblock \emph{Applied Numerical Mathematics}, 92:\penalty0 54--69, 2015.

\bibitem[Lu et~al.(2022)Lu, Ari{\~n}o~Bernad, and Solja{\v{c}}i{\'c}]{lu2022discovering}
Peter~Y Lu, Joan Ari{\~n}o~Bernad, and Marin Solja{\v{c}}i{\'c}.
\newblock Discovering sparse interpretable dynamics from partial observations.
\newblock \emph{Communications Physics}, 5\penalty0 (1):\penalty0 206, 2022.

\bibitem[Luo et~al.(2004)Luo, Strachan, and Swift]{luo2004nonequilibrium}
Sheng-Nian Luo, Alejandro Strachan, and Damian~C Swift.
\newblock Nonequilibrium melting and crystallization of a model lennard-jones system.
\newblock \emph{The Journal of chemical physics}, 120\penalty0 (24):\penalty0 11640--11649, 2004.

\bibitem[Luo et~al.(2020)Luo, Qin, Wan, Hu, and Yang]{luo2020parallel}
Zhaolong Luo, Xinming Qin, Lingyun Wan, Wei Hu, and Jinlong Yang.
\newblock Parallel implementation of large-scale linear scaling density functional theory calculations with numerical atomic orbitals in honpas.
\newblock \emph{Frontiers in Chemistry}, 8:\penalty0 589910, 2020.

\bibitem[M{\'e}l{\'e}ard(1996)]{Méléard1996}
Sylvie M{\'e}l{\'e}ard.
\newblock \emph{Asymptotic behaviour of some interacting particle systems; McKean-Vlasov and Boltzmann models}, pp.\  42--95.
\newblock Springer Berlin Heidelberg, Berlin, Heidelberg, 1996.
\newblock ISBN 978-3-540-68513-5.
\newblock \doi{10.1007/BFb0093177}.
\newblock URL \url{https://doi.org/10.1007/BFb0093177}.

\bibitem[Murray(2003)]{murray2003multi}
JD~Murray.
\newblock Multi-species waves and practical applications.
\newblock \emph{Mathematical Biology: II: Spatial Models and Biomedical Applications}, pp.\  1--70, 2003.

\bibitem[Musaelian et~al.(2023)Musaelian, Batzner, Johansson, and Kozinsky]{musaelian2023scaling}
Albert Musaelian, Simon Batzner, Anders Johansson, and Boris Kozinsky.
\newblock Scaling the leading accuracy of deep equivariant models to biomolecular simulations of realistic size.
\newblock In \emph{SC23: International Conference for High Performance Computing, Networking, Storage and Analysis}, pp.\  1--12. IEEE, 2023.

\bibitem[Ouala et~al.(2020)Ouala, Nguyen, Drumetz, Chapron, Pascual, Collard, Gaultier, and Fablet]{ouala2020learning}
Said Ouala, Duong Nguyen, Lucas Drumetz, Bertrand Chapron, Ananda Pascual, Fabrice Collard, Lucile Gaultier, and Ronan Fablet.
\newblock Learning latent dynamics for partially observed chaotic systems.
\newblock \emph{Chaos: An Interdisciplinary Journal of Nonlinear Science}, 30\penalty0 (10), 2020.

\bibitem[Park et~al.(2024)Park, Cheung, Choi, and Shin]{park2024tlasdi}
Jun Sur~Richard Park, Siu~Wun Cheung, Youngsoo Choi, and Yeonjong Shin.
\newblock tlasdi: Thermodynamics-informed latent space dynamics identification.
\newblock \emph{arXiv preprint arXiv:2403.05848}, 2024.

\bibitem[Ruelle \& Takens(1971)Ruelle and Takens]{ruelle1971nature}
David Ruelle and Floris Takens.
\newblock On the nature of turbulence.
\newblock \emph{Les rencontres physiciens-math{\'e}maticiens de Strasbourg-RCP25}, 12:\penalty0 1--44, 1971.

\bibitem[Samaey et~al.(2006)Samaey, Kevrekidis, and Roose]{samaey2006patch}
Giovanni Samaey, Ioannis~G Kevrekidis, and Dirk Roose.
\newblock Patch dynamics with buffers for homogenization problems.
\newblock \emph{Journal of Computational Physics}, 213\penalty0 (1):\penalty0 264--287, 2006.

\bibitem[Sauer et~al.(1991)Sauer, Yorke, and Casdagli]{sauer1991embedology}
Tim Sauer, James~A Yorke, and Martin Casdagli.
\newblock Embedology.
\newblock \emph{Journal of statistical Physics}, 65:\penalty0 579--616, 1991.

\bibitem[Schilders et~al.(2008)Schilders, Van~der Vorst, and Rommes]{schilders2008model}
Wilhelmus~HA Schilders, Henk~A Van~der Vorst, and Joost Rommes.
\newblock \emph{Model order reduction: theory, research aspects and applications}, volume~13.
\newblock Springer, 2008.

\bibitem[Schlaginhaufen et~al.(2021)Schlaginhaufen, Wenk, Krause, and Dorfler]{schlaginhaufen2021learning}
Andreas Schlaginhaufen, Philippe Wenk, Andreas Krause, and Florian Dorfler.
\newblock Learning stable deep dynamics models for partially observed or delayed dynamical systems.
\newblock \emph{Advances in Neural Information Processing Systems}, 34:\penalty0 11870--11882, 2021.

\bibitem[Shen \& Yang(2010)Shen and Yang]{shen2010numerical}
Jie Shen and Xiaofeng Yang.
\newblock Numerical approximations of allen-cahn and cahn-hilliard equations.
\newblock \emph{Discrete Contin. Dyn. Syst}, 28\penalty0 (4):\penalty0 1669--1691, 2010.

\bibitem[Stepaniants et~al.(2023)Stepaniants, Hastewell, Skinner, Totz, and Dunkel]{stepaniants2023discovering}
George Stepaniants, Alasdair~D Hastewell, Dominic~J Skinner, Jan~F Totz, and J{\"o}rn Dunkel.
\newblock Discovering dynamics and parameters of nonlinear oscillatory and chaotic systems from partial observations.
\newblock \emph{arXiv preprint arXiv:2304.04818}, 2023.

\bibitem[Takens(2006)]{takens2006detecting}
Floris Takens.
\newblock Detecting strange attractors in turbulence.
\newblock In \emph{Dynamical Systems and Turbulence, Warwick 1980: proceedings of a symposium held at the University of Warwick 1979/80}, pp.\  366--381. Springer, 2006.

\bibitem[Thompson et~al.(2022)Thompson, Aktulga, Berger, Bolintineanu, Brown, Crozier, in~'t Veld, Kohlmeyer, Moore, Nguyen, Shan, Stevens, Tranchida, Trott, and Plimpton]{LAMMPS}
A.~P. Thompson, H.~M. Aktulga, R.~Berger, D.~S. Bolintineanu, W.~M. Brown, P.~S. Crozier, P.~J. in~'t Veld, A.~Kohlmeyer, S.~G. Moore, T.~D. Nguyen, R.~Shan, M.~J. Stevens, J.~Tranchida, C.~Trott, and S.~J. Plimpton.
\newblock {LAMMPS} - a flexible simulation tool for particle-based materials modeling at the atomic, meso, and continuum scales.
\newblock \emph{Comp. Phys. Comm.}, 271:\penalty0 108171, 2022.
\newblock \doi{10.1016/j.cpc.2021.108171}.

\bibitem[Vollmayr-Lee(2020)]{vollmayr2020introduction}
Katharina Vollmayr-Lee.
\newblock Introduction to molecular dynamics simulations.
\newblock \emph{American Journal of Physics}, 88\penalty0 (5):\penalty0 401--422, 2020.

\bibitem[Wainwright(2019)]{Wainwright_2019}
Martin~J. Wainwright.
\newblock \emph{Uniform laws of large numbers}, pp.\  98–120.
\newblock Cambridge Series in Statistical and Probabilistic Mathematics. Cambridge University Press, 2019.

\bibitem[Wang et~al.(2019)Wang, Olsson, Wehmeyer, P{\'e}rez, Charron, De~Fabritiis, No{\'e}, and Clementi]{wang2019machine}
Jiang Wang, Simon Olsson, Christoph Wehmeyer, Adri{\`a} P{\'e}rez, Nicholas~E Charron, Gianni De~Fabritiis, Frank No{\'e}, and Cecilia Clementi.
\newblock Machine learning of coarse-grained molecular dynamics force fields.
\newblock \emph{ACS central science}, 5\penalty0 (5):\penalty0 755--767, 2019.

\bibitem[Yang et~al.(2023)Yang, Li, Lee, Choi, and Kim]{yang2023fast}
Junxiang Yang, Yibao Li, Chaeyoung Lee, Yongho Choi, and Junseok Kim.
\newblock Fast evolution numerical method for the allen--cahn equation.
\newblock \emph{Journal of King Saud University-Science}, 35\penalty0 (1):\penalty0 102430, 2023.

\bibitem[Yu et~al.(2021)Yu, Tian, E, and Li]{yu2021onsagernet}
Haijun Yu, Xinyuan Tian, Weinan E, and Qianxiao Li.
\newblock Onsagernet: Learning stable and interpretable dynamics using a generalized onsager principle.
\newblock \emph{Physical Review Fluids}, 6\penalty0 (11):\penalty0 114402, 2021.

\bibitem[Zhang et~al.(2018)Zhang, Han, Wang, Car, et~al.]{zhang2018deepcg}
Linfeng Zhang, Jiequn Han, Han Wang, Roberto Car, et~al.
\newblock Deepcg: Constructing coarse-grained models via deep neural networks.
\newblock \emph{The Journal of chemical physics}, 149\penalty0 (3), 2018.

\bibitem[Zhang et~al.(2019)Zhang, Lin, Wang, Car, and E]{zhang2019active}
Linfeng Zhang, De-Ye Lin, Han Wang, Roberto Car, and Weinan E.
\newblock Active learning of uniformly accurate interatomic potentials for materials simulation.
\newblock \emph{Physical Review Materials}, 3\penalty0 (2):\penalty0 023804, 2019.

\bibitem[Zhang et~al.(2022)Zhang, Shin, and Em~Karniadakis]{zhang2022gfinns}
Zhen Zhang, Yeonjong Shin, and George Em~Karniadakis.
\newblock Gfinns: Generic formalism informed neural networks for deterministic and stochastic dynamical systems.
\newblock \emph{Philosophical Transactions of the Royal Society A}, 380\penalty0 (2229):\penalty0 20210207, 2022.

\bibitem[Zhou \& Liu(2022)Zhou and Liu]{ZHOU202241}
Kun Zhou and Bo~Liu.
\newblock Chapter 2 - potential energy functions.
\newblock In Kun Zhou and Bo~Liu (eds.), \emph{Molecular Dynamics Simulation}, pp.\  41--65. Elsevier, 2022.

\end{thebibliography}
